\documentclass[acmsmall,authorversion,nonacm]{acmart}
\AtBeginDocument{%
  }

\setcopyright{acmlicensed}
\copyrightyear{2024}
\acmYear{2024}
\acmDOI{XXXXXXX.XXXXXXX}

\acmJournal{JACM}
\acmVolume{37}
\acmNumber{4}
\acmArticle{111}
\acmMonth{7}

\usepackage{graphicx} 
\usepackage{amsfonts} 
\usepackage{amsmath}
\usepackage{amsthm}
\usepackage{hyperref}

\usepackage{listings}
\usepackage{color}
\usepackage{subfigure}
\usepackage{tikz}
\usepackage{xcolor}
\usepackage{colortbl}
\usepackage{natbib}
\usepackage{booktabs}
\usepackage{tabularx}
\usepackage{mysymbol}
\usepackage{ifthen}
\usepackage{diagbox}
\usepackage{multirow}
\usepackage{pifont}
\newtheorem{assumption}{Assumption}
\newtheorem{problem}{Problem}
\newtheorem{remark}{Remark}
\usepackage{bm}

\newcommand{\md}[1]{\bbF_{\text{#1}}}
\newcommand{\gti}{\bbv_{k}}

\newcommand{\bhline}{\noalign{\hrule height 1.2pt}}

\newcommand{\modifycolor}{black}  
\newcommand{\mv}{\texttt{MV.jl}\ }

\newcommand{\circled}[1]{\tikz[baseline=(char.base)]{
            \node[shape=circle,draw,inner sep=0.8pt,font=\scriptsize] (char) {#1};}}
\graphicspath{{imgs/}}

\begin{document}

\title{Certifying Robustness of Learning-Based Keypoint Detection and Pose Estimation Methods}

\author{Xusheng Luo}
\email{xushengl@andrew.cmu.edu}
\orcid{0000-0003-4342-7234}
\author{Tianhao Wei}
\email{twei2@andrew.cmu.edu}
\orcid{0000-0003-2505-4585}
\author{Simin Liu}
\email{siminliu@andrew.cmu.edu}
\orcid{0009-0009-1495-6920}
\author{Ziwei Wang}
\email{ziweiwa2@andrew.cmu.edu}
\orcid{0000-0001-9225-8495}
\affiliation{%
  \institution{Carnegie Mellon University}
  \streetaddress{5000 Forbes Avenue}
  \city{Pittsburgh}
  \state{Pennsylvania}
  \country{USA}
  \postcode{15213}
}

\author{Luis Mattei-Mendez}
\email{Luis.E.Mattei-Mendez@boeing.com}
\author{Taylor Loper}
\email{taylor.s.loper@boeing.com}
\author{Joshua Neighbor}
\email{joshua.neighbor@boeing.com}
\affiliation{%
  \institution{The Boeing Company}
  \streetaddress{PO Box 3707}
  \city{Seattle}
  \state{Washington}
  \country{USA}
  \postcode{98214-2207}
}

\author{Casidhe Hutchison}
\email{fhutchin@nrec.ri.cmu.edu}
\author{Changliu Liu}
\email{cliu6@andrew.cmu.edu}
\orcid{0000-0002-3767-5517}
\affiliation{%
  \institution{Carnegie Mellon University}
  \streetaddress{5000 Forbes Avenue}
  \city{Pittsburgh}
  \state{Pennsylvania}
  \country{USA}
  \postcode{15213}
}

\renewcommand{\shortauthors}{Luo et al.}

\begin{abstract}
  This work addresses the certification of the local robustness of vision-based two-stage 6D object pose estimation. The two-stage method for object pose estimation achieves superior accuracy over the single-stage approach by first employing deep neural network-driven keypoint regression and then applying a Perspective-n-Point (PnP) technique.
    Despite advancements, the certification of these methods' robustness, especially in safety-critical scenarios, remains scarce. This research aims to fill this gap with a focus on their local robustness on the system level—the capacity to maintain robust estimations amidst semantic input perturbations. The core idea is to transform the certification of local robustness into a process of neural network verification for classification tasks. The challenge is to develop model, input, and output specifications that align with off-the-shelf  verification tools. To facilitate verification, we modify the keypoint detection model by substituting nonlinear operations with those more amenable to the verification processes. Instead of merely injecting random noise into images, as is common, we employ a convex hull representation of images as input specifications to more accurately depict semantic perturbations. Furthermore, by conducting a sensitivity analysis, we propagate the robustness criteria from pose estimation to keypoint accuracy, and then formulating an optimal error threshold allocation problem that allows for the setting of a maximally permissible keypoint deviation thresholds. Viewing each pixel as an individual class, these thresholds result in linear, classification-akin output specifications. Under certain conditions, we demonstrate that the main components of our certification framework are both sound and complete, and validate its effects through extensive evaluations on realistic perturbations. To our knowledge, this is the first study to certify the robustness of large-scale, keypoint-based pose estimation given images in real-world scenarios. 
\end{abstract}


\maketitle

\section{Introduction}
In the realm of computer vision, vision-based 6D object pose estimation, i.e., 3D rotation and 3D translation of an object with respect to the camera,  serves as a pivotal method for identifying, monitoring, and interpreting the posture and movements of objects through images~\cite{fan2022deep,thalhammer2023challenges}
This technology is fundamental in granting machines the ability to comprehend the physical environment, finding its utility in diverse domains such as robotics~\cite{collet2011moped,deng2020self}, augmented reality~\cite{tang20193d}, and human-computer interaction~\cite{zheng2023deep}. The evolution of deep learning and the adoption of neural networks, particularly convolutional neural networks (CNNs), have markedly surpassed traditional techniques that depend on manually engineered features. Within the spectrum of learning-based approaches, a distinct classification exists: single-stage methods directly estimate the 6D pose from an image~\cite{kehl2017ssd,di2021so,zheng2023hs}. Conversely, a more widespread and accurate category of methods employs a two-stage strategy, initially regressing sparse keypoints~\cite{rad2017bb8,oberweger2018making,he2021ffb6d} or dense pixels~\cite{wang2019normalized,park2019pix2pose,peng2019pvnet,shugurov2021dpodv2,lian2023checkerpose} from the image, followed by the utilization of a Perspective-n-Point (PnP)-based strategy for pose estimation through established 3D-2D point correspondences.

Despite the increasing efforts to boost the empirical robustness of these methods against challenges like occlusions, fluctuating lighting, and varied backgrounds, the focus on validating or certifying the reliability of vision-based pose estimation systems remains minimal. The absence of performance assurances for these frameworks raises concerns about their integration into safety-critical applications. In this work, our objective is to certify the robustness of learning-based keypoint detection and pose estimation approaches given input images. We focus on the aspect of local robustness, which refers to the ability to maintain consistent performance or predictions when the input data is perturbed around a given input point. The core question is determining whether pose estimation stays within an acceptable range when the input image undergoes perturbations. To the best of our knowledge, this study is the first one to certify the robustness of large-scale, keypoint-based pose estimation problem encountered in the real world.

Given the crucial role of neural networks in learning-based visual pose estimation, certifying their robustness is inherently linked to neural network verification~\cite{liu2021algorithms}. This area of verification has attracted significant attention in recent years, driven by the paradox of widespread adoption of neural network solutions without adequate assurance of their reliability, primarily due to their opaque nature~\cite{brix2023fourth}. What distinguishes our problem from existing neural network (NN) verification efforts is our focus on system-level properties rather than on verifying attributes of isolated neural networks, which typically relate directly to the NN outputs~\cite{katz2019marabou,bak2020improved,wang2021beta}. For example, in robustness verification of classification models, the objective is to ensure the predicted class remains unchanged despite variations in input. Our research, however, targets the pose estimation framework, wherein the NN constitutes only one component of the entire system. The verification encompasses system-wide requirements, necessitating not just the evaluation of the keypoint detection model's robustness but also that of the PnP-based method.

\begin{figure}
    \centering
    \includegraphics[width=1\linewidth]{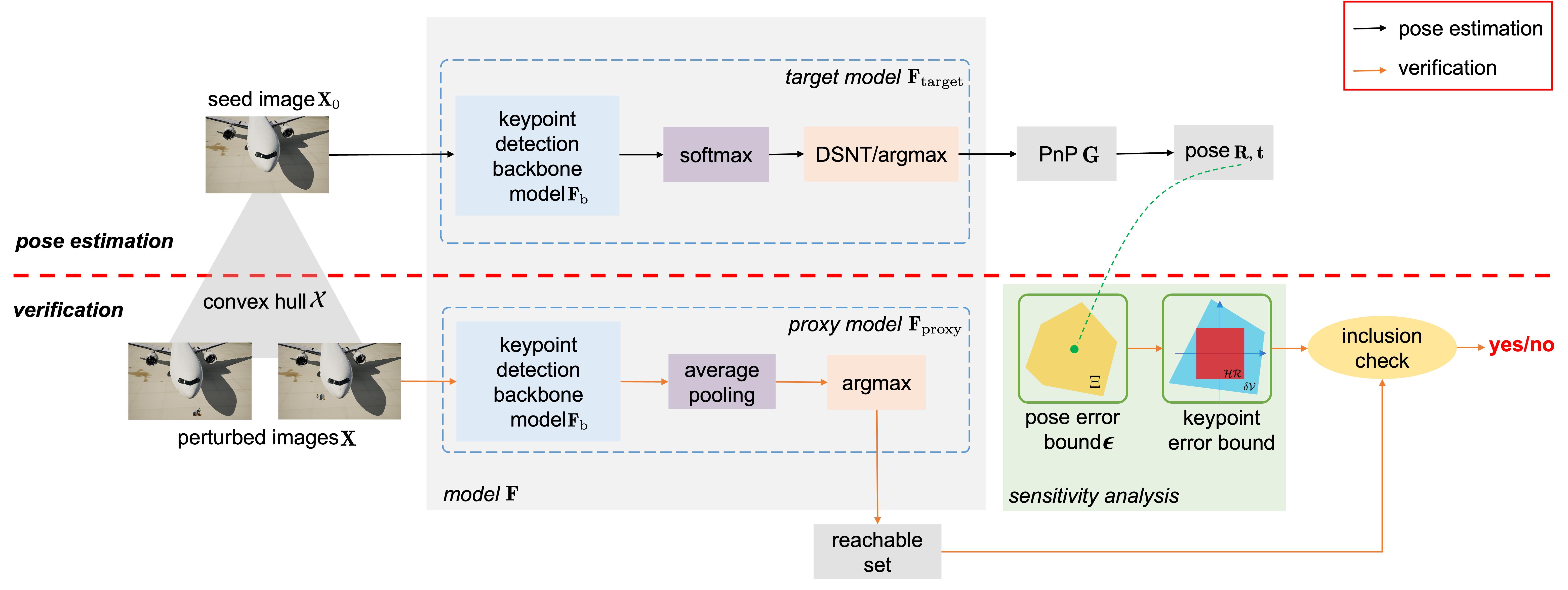}
    \caption{{\color{\modifycolor}Overview of the PnP-based pose estimation and the proposed verification framework. A  thick {\color{red}red} dashed line divides the sections for pose estimation (above) and verification (below). For pose estimation, a seed image $\bbX_0$ is processed by the target model $\bbF_{\text{target}}$ to identify keypoints, which are then input into the PnP method $\bbG$ to determine the pose $\bbR$ and $\bbt$. The verification framework takes as input the seed image $\bbX_0$ and a set of perturbed images $\bbX$ that form the convex hull $\ccalX$, along with the pose error bound $\bm{\epsilon}$. Through sensitivity analysis, this pose error bound is transformed into a keypoint error bound, which in turn determines the parameters of the average pooling operation. This substitution replaces the less verification-friendly softmax operation, creating the proxy model $\bbF_{\text{proxy}}$. 
    By checking the inclusion relation between the reachable set of model $\bbF_{\text{proxy}}$ and the output specification, the verification tool returns whether the model is robustness.
    }}
    \label{fig:overview}
\end{figure}

\subsection{Overview of the Approach}
To certify the robustness of the two-stage keypoint-based pose estimation framework, our main idea is to convert the local robustness certification of pose estimation into a standard neural network verification problem for classification networks. The primary challenge involves crafting three verification components—model, input, and output specifications—in a manner that is compatible with existing verification tools. {\color{\modifycolor}A graphical overview is provided in Fig.~\ref{fig:overview}. }
\paragraph{Model modification for verification.} Considering the keypoint detection model involves complex nonlinear operations such as the softmax function, a variant of the model that is more amenable to verification is created, and an analysis is conducted to understand the properties that are maintained between the original and this modified model. 
\paragraph{Input specification through convex hulls.} To account for realistic semantic perturbations in input images, we define the input space as a convex hull of possible perturbed images, which captures variations in a mathematically rigorous manner, allowing for a linear representation of input perturbations. Such a specification outperforms existing methods that simply introduce random independent noise into images.
\paragraph{Output specification via sensitivity analysis.} The core of connecting system requirements with the neural network's output lies in conducting a sensitivity analysis of the downstream PnP method. By understanding how variations in detected keypoints affect the estimated pose, it's possible to translate system-level pose accuracy requirements into error thresholds for keypoint detection. By treating each pixel as a separate class, these thresholds are then used to define classification-like linear output specifications.

\subsection{Contributions}
Our contributions can be summarized as follows:
\begin{enumerate}
    \item We propose a local robustness certification framework for the learning-based keypoint detection and pose estimation pipeline;
    \item We analyse the soundness and completeness properties of this certification framework;
    \item We demonstrate the method's efficacy through validation on a real-world scale keypoint-based pose estimation problem.
\end{enumerate}

\section{Related Work}

\subsection{Formal Verification of Neural Networks}
The objective of verifying neural networks involves ensuring they meet certain standards of safety, security, accuracy, or robustness. This essentially means determining the truth of a specific claim about the outputs of a network based on its inputs. In recent years, there has been a significant influx of research in this area. For comprehensive insights into neural network verification, one can refer to~\cite{liu2021algorithms}. Verification techniques are generally divided into three main groups: reachability-based approaches, which perform a layer-by-layer analysis to assess network output range~\cite{gehr2018ai2,xiang2018output,tran2020nnv}; optimization methods, which seek to disprove the assertion~\cite{bastani2016measuring,tjeng2018evaluating}; and search-based strategies which combine with reachability analysis or optimization to identify instances that contradict the assertion~\cite{katz2019marabou,xu2020fast,wu2024marabou,duong2024harnessing}. In 2020, VNN-COMP~\cite{brix2023fourth} launched as a competition to evaluate the capabilities of advanced verification tools spanning a variety of tasks, including collision detection, image classification, dataset indexing, and image generation. However, these methods treat deep neural networks in isolation, concentrating on analyzing the input-output relationship.

Concurrently, there is research focused on the system-level safety of cyber-physical systems (CPS) incorporating neural network components, particularly within the system and controls domain. They broadly fall into two categories. The first category~\cite{tran2019safety,dutta2019reachability,everett2021neural,ivanov2021verisig} focuses on ensuring the correctness of neural network-based controllers, taking their input from the structured outcomes of the state estimation module, regardless of the state estimation module is based on perception or not. Neural network controllers of this type generally consist of several fully connected layers, making them relatively straightforward to verify. The second category focuses on validating the closed-loop performance of vision-based dynamic systems that incorporate learning-based components. Among these, studies such as~\cite{sun2019formal,ivanov2020case,ivanov2021compositional,hsieh2022verifying,sun2022formal} examine LiDARs as the perception module, processed by multi-linear perceptrons (MLPs) with a few hidden layers. Other approaches, primarily applied to runway landing and lane tracking, deal with high-dimensional inputs from camera images, employing methods like approximate abstraction of the perception model~\cite{hsieh2022verifying}, contract synthesis~\cite{astorga2023perception}, simplified networks within the perception model~\cite{cheng2020towards,katz2022verification}, or a domain-specific model of the image formation process~\cite{santa2022nnlander}. 
{\color{\modifycolor}Nevertheless, studies directly dealing with high-dimensional inputs from camera images are still limited due to the images' high dimensionality and unstructured data nature, in contrast to structured robot states such as position and velocity.}


\subsection{Certification of Keypoint Detection and Pose Estimation Methods}
The investigation of certification methods for pose estimation is relatively limited.~\cite{talak2023certifiable} introduced a certifiable approach to keypoint-based pose estimation from point clouds by correcting keypoints identified by the model, ensuring the correctness guarantee of the pose estimation.~\cite{shi2023correct} expanded on this by integrating the correction concept with ensemble self-training. In a similar vein, By propagating the uncertainty in the keypoints to the object pose,~\cite{yang2023object} created a keypoint-based pose estimator for point clouds that is provably correct and is characterized by definitive worst-case error bounds. The above work focuses on point clouds as opposed to images.~\cite{wursthorn2024uncertainty} applied an advanced deep learning technique for uncertainty quantification to assess the uncertainty (i.e., the predicted distribution of a pose) in multi-stage 6D object pose estimation methods. In contrast,~\cite{santa2023certified} aimed to design a neural network that processes camera images to directly predict the aircraft's position relative to the runway with certifiable error bounds. Of all these studies,~\cite{kouvaros2023verification} is the most similar to our work, which focuses on the verification of keypoint detection, excluding the examination of the PnP method for system-wide assurances. Their approach verifies the neural network in isolation and is limited to very slight perturbations, failing to encompass realistic semantic variations.

\section{Background}
In this work, we represent scalars and scalar functions by italicized lowercase letters ($x$), vectors and vector functions by upright bold lowercase letters ($\bbx$), matrices and matrix functions by upright bold uppercase letters ($\bbX$), and sets and set functions with calligraphic uppercase letters ($\ccalX$).

\subsection{Keypoint-based Pose Estimation}\label{sec:bg_pose}
The keypoint-based approach consists of two steps to estimate the 6D pose from a 2D model image. First, a neural network is tasked with predicting the 2D locations of keypoints, whose 3D locations are predefined within the object model. A common strategy involves the use of heatmap regression, wherein ground-truth heatmaps are created by placing 2D Gaussian kernels atop each keypoint. The heatmap pixel values are interpreted as the likelihood of each pixel being a keypoint. These heatmaps are then used to guide the training through an $\ell_2$ loss. The detection network can be divided into two parts. A \textit{backbone} network, denoted by $\md{b}$, inputs a 2D image to produce unnormalized heatmaps, which is first followed by a softmax layer that transforms unnormalized heatmaps into normalized ones, and then by argmax operations, or another layer of differentiable spatial to numerical transformation (DSNT)~\cite{nibali2018numerical}  for keypoint extraction. We refer to the part after the softmax (including) as the \textit{head} network. The entire network is represented by
$   \bbV =  \bbF(\bbX) = \md{h}  \circ \md{b}(\bbX)$, 
where $\bbX \in \mbR^{H\times W\times C}$ represents a 2D RGB image with dimensions being $H\times W\times C$, and $\bbV \in \mbR^{K\times 2}$ denotes the 2D coordinates of $K$ keypoints. Here, $\circ$ denotes function composition. To enhance accuracy and robustness, it's often essential to preprocess the input image $\bbX$ before it is passed to the network, such as resizing and color normalization. Denote this preprocessing step by $\bbF_0$, leading to the equation $  \bbV =  \bbF(\bbX) = \md{h}  \circ \md{b} \circ \bbF_0(\bbX)$.  In what follows, we omit the preprocessing step unless it is critical to consider it.

\begin{figure}[!t]
    \centering
    \includegraphics[width=\linewidth]{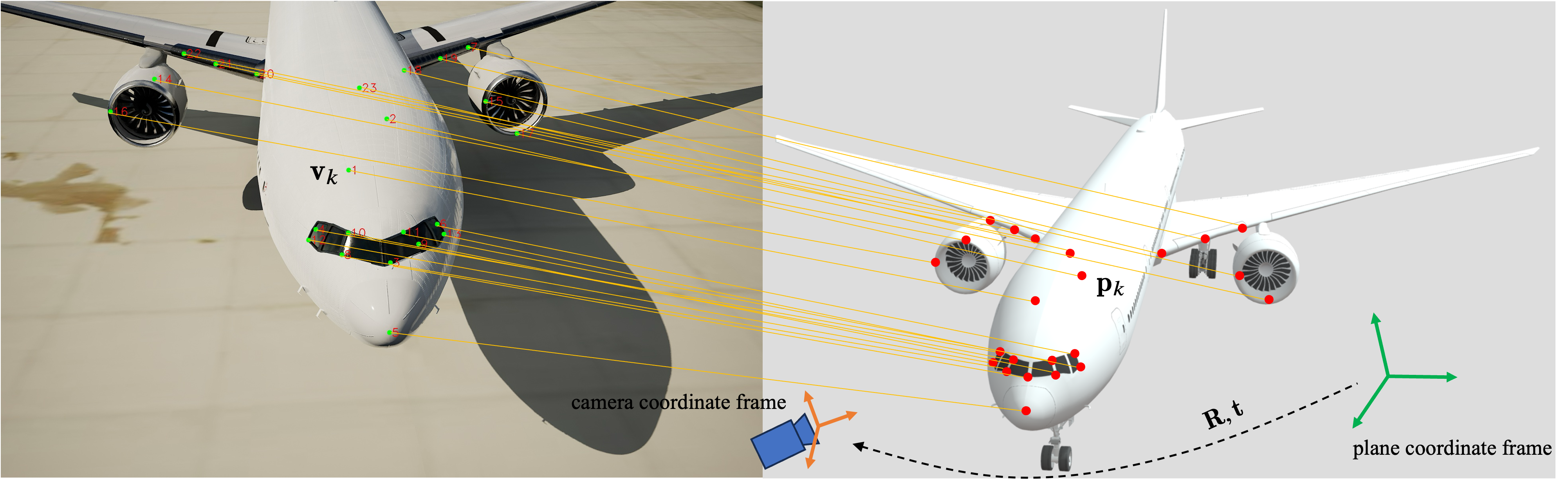}
    \caption{{\color{\modifycolor}Pose estimation of an airplane parked at airports is conducted using a PnP-based method. The method uses 23  keypoints, marked in {\color{red}red}, which are placed across the airplane's surface to thoroughly cover the aircraft's body, as shown in the 3D model from~\cite{hikami3150_2024} (right). These keypoints have predefined 3D coordinates within the airplane's coordinate system. An overhead image of the airplane is taken and 2D keypoints, marked in {\color{green}green}, are identified through a keypoint detection network. The PnP-based method computes the transformation matrix between the plane and camera coordinate frames.}}
    \label{fig:airplane}
    \vspace{-10pt}
\end{figure}

The second step employs the Perspective-n-Point (PnP) algorithm, which executes a nonlinear least squares (NLS) optimization to estimate the 6D pose from established 3D-to-2D correspondences. An illustration of the pose estimation  for an airplane is presented in Fig.~\ref{fig:airplane}. Let $\bbK \in \mbR^{3\times3}$ denote the camera intrinsic parameter matrix, $\bbp_k \in \mbR^3$ denote the 3D coordinate of keypoint $k$, where $k = 1, \ldots, K$, and $\bbv_k \in \mbR^2$ denote the corresponding 2D coordinate. These 3D-2D correspondences are formed through the following perspective projection model:
\begin{align}\label{eq:projection}
   \lambda_k \begin{bmatrix} 
    \bbv_k \\ 
    1
   \end{bmatrix}
   = \bbK (\bbR \bbp_k + \bbt),
\end{align}
where $\bbR \in \mbR^{3\times3}$ and $\bbt \in \mbR^{3}$ represent the rotation matrix and translation vector, respectively, that establish the transformation between the object and camera coordinate systems, with $\lambda_k$ representing the scaling factor. The objective of the PnP method is to approximate this transformation, represented by $\mathbf{\hhatR}$ and $\mathbf{\hhatt}$, by minimizing the $\ell_2$ norm of the reconstruction error:
\begin{equation}\label{eq:nls}
\begin{aligned}
\langle \mathbf{\hhatR}, \mathbf{\hhatt} \rangle\; =  \; \underset{\bbR, \bbt}{\argmin}
& & \sum_{k=1}^{K} \left\| \bbK (\bbR \bbp_k + \bbt) -   \lambda_k \begin{bmatrix} 
    \bbv_k \\ 
    1
   \end{bmatrix}\right\|^2_2.
\end{aligned}
\end{equation}

Let the notation $\langle \mathbf{\hhatR}, \mathbf{\hhatt} \rangle = \bbG(\bbP, \bbV)$ represent the PnP procedure, where $\bbP \in \mbR^{K\times 3}$ denotes the 3D coordinates of keypoints. Note that the keypoint-based pose estimation framework described above is the basic version. Numerous adaptations have been developed to increase accuracy and robustness~\cite{liu2024deep}, particularly in challenging conditions such as occlusions, varying viewpoints, and different lighting scenarios. This paper concentrates on validating the robustness of this foundational pipeline, marking a crucial step towards certifying the effectiveness of more intricate keypoint-based pose estimation techniques. Hereafter, we use $\bm{\Psi}$ to represent this pipeline, which includes keypoint detection followed by the application of a PnP method, that is, $\langle \mathbf{\hhatR}, \mathbf{\hhatt} \rangle = \bm{\Psi}(\bbX) = \bbG(\bbP, \bbF(\bbX))$.


\subsection{Verification of Neural Networks}
Consider a multi-layer neural network representing a function $\bbf$, which takes an input $\bbx \in  \ccalD_\bbx \subseteq \mbR^{d_0}$ and produces an output $\bby \in \ccalD_\bby \subseteq \mbR^{d_n}$, where $d_0$ is the input dimension, and $d_n$ is the output dimension. Any non-vector inputs or outputs are restructured into vector form. The verification process entails assessing the validity of the following input-output relationships defined by the function $\bbf$:
$
    \bbx \in \ccalX \Rightarrow \bby = \bbf(\bbx) \in \ccalY,
$
where sets, $\ccalX \subseteq \ccalD_\bbx $ and $\ccalY \subseteq \ccalD_\bby$, are referred to as input and output constraints, respectively. 

In the context of confirming the robustness of a classification network, the goal is to ascertain that all samples within a proximal vicinity of a specified input $x_0$ receive an identical classification label. Assuming the target label is $i^\ast \in \{1, \ldots, d_n\}$, the specification for verification is that $y_{i^\ast} > y_j$ for every $j$ not equal to $i^\ast$. The constraints on inputs and outputs are established accordingly:
$
\ccalX = \{\bbx \mid \|\bbx - \bbx_0\|_p \leq \epsilon\}, 
\ccalY = \{\bby \mid y_{i^\ast} > y_j, \; \forall j \neq i^\ast\},
$
where $\epsilon$ represents the maximum permissible deviation in the input space. The metric used to quantify disturbance can be any $\ell_p$ norm.

Neural network verification algorithms can generally be categorized into three main types: reachability analysis, optimization, and search. {\color{\modifycolor}NN verification essentially seeks to transform the nonlinear model checking problem into piece-wise linear satisfiability problems, and it can be applied to various nonlinearities, including ReLU and, more recently, softmax~\cite{pmlr-v206-wei23c}. Two pivotal attributes—\textit{soundness} and \textit{completeness}—are of critical importance. A verification algorithm is \textit{sound} if it only confirms the validity of a property when the property is indeed valid. It is \textit{complete} if it consistently recognizes and asserts the existence of a property whenever it is actually present. There is a trade-off between computational complexity and conservativeness (or in-completeness).}


\section{Problem Formulation}

The most common type of input specification involves limiting the $\ell_p$-norm of the variation to a threshold, that is, $\| \bbX - \bbX_0 \|_{p} \leq \epsilon$. {\color{\modifycolor}However, $\ell_p$ perturbation,  often characterized by a small value $\epsilon$,  is not a correct mathematical description of more realistic perturbations, as the independent nature of pixel-wise perturbations falls short in creating perturbations that reflect semantic correlations between pixels, such as large variations in lighting, weather conditions, and the effects of camera motion blur. Another approach for input perturbations is to directly add a generative model that perturbs the input image before the original neural network, such as~\cite{poursaeed2018generative}, and then verify them together. However, the verification result can be biased by the generative model used and it is not user friendly due to the difficulty in controlling the changes made to the image.} To address these shortcomings, we adopt a strategy based on the convex hull, which involves combining a seed image with a collection of perturbed images. These perturbed images can be created through different methods including simulators and learning-based generative models. Also, the convex hull specification can directly enable users to specify the perturbed images, which makes the perturbation specifications more user friendly.

\begin{definition}[Convex hull of images]
Given a seed image $\bbX_0$ and a set of $n$ perturbed images $\{\bbX_1, \ldots, \bbX_n\}$, the convex hull constituted by these images is defined by the set of all their possible convex combinations. Mathematically,
\begin{align}
 \ccalX = \left\{  \bbX \;\left\vert \; \bbX = \sum_{i = 0}^n \alpha_i \bbX_i, \quad \text{s.t.} \; \alpha_i \geq 0, \sum_{i=0}^n \alpha_i  = 1\right.\right\}.
\end{align}
\end{definition}
Images within the convex hull $\bbX \in \ccalX$ result from varying degrees of continuous blending among the provided images.  Through the convex combination of perturbation instances, we can model environmental and sensor-related perturbations, including changes in brightness, contrast, weather conditions, motion blur, and dust on lens. Convex hull perfectly captures the entire range of brightness or contrast shifts, and does a reasonable approximation of color shifts, provided they are small. However, this approach does not capture translational perturbations of the object or perturbations related to camera movements, such as changes in the viewpoint. Preliminary methods exist that characterize pixel-wise perturbations caused by camera motion~\cite{hu2023robustness}, but these methods cover only a very limited range. Investigating robustness against such perturbations will be addressed in future research.  A collection of perturbed images and sampled images from the convex hull can be found in Fig.~\ref{fig:airplanes} in Section~\ref{sec:verif_res}. 

{\color{\modifycolor}Given two rotation matrices $\bbR$ and $\mathbf{\hhatR}$, we define the function $\bbD_r (\mathbf{\hhatR}, \bbR) = \vert \bbr(\mathbf{\hhatR}) - \bbr(\bbR) \vert$, where $\bbr(\cdot)$ represents the function that converts rotations to Euler angles, and $\vert \cdot \vert$ denotes the element-wise absolute value. Thus, $\bbD_r$ measures the rotational difference along each axis. Similarly, define $\bbD_t \left(\mathbf{\hhatt}, \bbt \right) =  \left\vert  \mathbf{\hhatt} - \bbt \right \vert$, which calculates the translational difference per axis.}


\begin{problem}\label{prob:certification}
 {\color{\modifycolor}Given a convex hull representation $\ccalX$ consisting of a seed image $\bbX_0$ and $n$ perturbed images, and a keypoint-based pose estimation framework $\bm{\Psi}$.
 The problem is to certify whether the pose estimation framework $\bm{\Psi}$ is robust to any image within the set $\ccalX$.} Mathematically,
\begin{align}\label{eq:problem}
\begin{bmatrix}
    \bbD_r \left(\mathbf{\hhatR}, \bbR \right ) \\
    \bbD_t \left(\mathbf{\hhatt}, \bbt \right)
\end{bmatrix} \leq 
\begin{bmatrix}
    \bm{\epsilon}_r \\
    \bm{\epsilon}_t
\end{bmatrix}\quad \text{s.t.} \;  \langle \mathbf{\hhatR}, \mathbf{\hhatt} \rangle = \bm{\Psi}(\bbX), \;\forall\, \bbX \in \ccalX,
\end{align}
{\color{\modifycolor}where $\langle \bbR, \bbt \rangle =  \bbG(\bbP, \bbV)$ represents the pose estimated given the ground-truth 3D keypoint coordinates $\bbP$ and ground-truth 2D keypoint coordinates  $\bbV$ for $\bbX_0$,} and $\bm{\epsilon}_r \in \mbR^3$ and $\bm{\epsilon}_t \in \mbR^3$ denote the error thresholds specified by the user for rotation and translation. We refer to $\bbR$ and $\bbt$ as the \textit{nominal} transform, and $ \mathbf{\hhatR}$ and $ \mathbf{\hhatt}$ as the \textit{perturbed} transform. In simpler terms to describe Eq.~\eqref{eq:problem}, the variation between the nominal transform and the perturbed transform is kept within these predefined thresholds.
\end{problem}

\begin{remark}
If a pre-processing step, such as resizing or color normalization, is applied to an image before it is input into the keypoint detection model, then the convex hull should be constructed using images that have also undergone these pre-processing steps. Consequently, any image within the convex hull is represented as $\bbX = \sum_{i=0}^n \alpha_i \bbF_0(\bbX_i)$, where $\bbF_0$ denotes the pre-processing function. 
\end{remark}

 The vector of error thresholds $\bm{\epsilon_r}$ and $\bm{\epsilon_t}$ can specify requirements from system designers in an interpretable manner. For example, such requirement can be setting the error thresholds for translation to be within 1 meter. The difficulty presented by Problem~\ref{prob:certification} is that, the error thresholds $\bm{\epsilon}$ are defined with respect to pose estimation rather than the direct outputs of neural networks—keypoints—and there are nonlinear mappings between keypoints and pose estimation. 
To address Problem~\ref{prob:certification}, the core of our approach is to  transform the local robustness certification problem into a standard neural network verification problem for classification networks. This involves adapting three verification components—model, input, and output specifications—to be compatible with existing verification tools. As detailed in Section~\ref{sec:model}, we have modified the keypoint detection model to handle complex operations like the softmax function, simplifying the verification process. The input space is the a convex hull of possible perturbed images, which captures semantic variations  compared to traditional methods that introduce random noise. Section~\ref{sec:output_spec} outlines how the output specification is established through a sensitivity analysis of the PnP method, detailed further in Section~\ref{sec:error_prop_pooling}. This analysis helps translate system-level pose accuracy requirements into error thresholds for keypoint detection. Finally, Section~\ref{sec:error_analysis} is dedicated to analyzing the soundness and completeness of our proposed verification framework.


\section{Formulation of Verification of the Neural Network}

{\color{\modifycolor}To convert Problem~\ref{prob:certification} into an equivalent problem that can be verified using existing NN verification tools, we introduce approaches to modify the NN model and specify the output constraints.}

\subsection{Verification-Friendly Keypoint Detection Model}\label{sec:model}

As introduced in Section~\ref{sec:bg_pose}, a softmax layer follows the backbone network to generate probabilistic heatmaps. This step introduces verification challenges due to the highly nonlinear exponential function in softmax. To address this, we propose to maintain the backbone network but replace the head network with an alternative one. A critical insight is that during the inference phase, the unnormalized heatmaps output by the backbone network $\md{b}$ already provide significant clues about the locations of keypoints. The subsequent operations within the head network $\md{h}$, such as softmax and DSNT, essentially serve to aggregate this information across the entire heatmap. This aggregation allows for the generation of predictions in an average manner. In essence, the peaks within the heatmaps play a pivotal role in these predictions, and their locations remain unchanged between unnormalized and normalized heatmaps. Consequently, by focusing on the characteristics of these peaks, we can bypass some of the complexities introduced by the head network's nonlinear operations while still capturing the essential information required for accurate keypoint detection. 

By appending an average pooling layer followed by an argmax layer as the new head network, referred to as $\md{h}'$, to the backbone network, we create a \textit{proxy} model, denoted by $\md{proxy} = \md{h}' \circ \md{b}$. We refer to the original model as the \textit{target} model. The introduction of this proxy model simplifies the transformation of the verification problem into one akin to classification, where each pixel is treated as a unique category, as illustrated in Fig.~\ref{fig:overview}. The average pooling layer offers two advantages. Firstly, its downsampling effect significantly reduces the number of categories (pixels), thereby simplifying the complexity of verification. Secondly, by aggregating local features within each pooling region, it ensures a more accurate representation than applying argmax directly. 

The selection of pooling parameters should be guided by the error threshold on pose estimation $\bm{\epsilon}_r$ and $\bm{\epsilon}_t$, as elaborated in Section~\ref{sec:error_prop_pooling}. This leads to a scenario where the pooling parameters assigned to each heatmap, corresponding to individual keypoints, vary. This variation arises because keypoints have differing levels of influence on the accuracy of pose estimation. Consequently, keypoints with lesser influence on pose estimation are allocated a larger permissible error range, whereas those that play a critical role in pose estimation are subjected to stricter error ranges. It is worth noting that, with certain assumptions, the target model and proxy model are equivalent in terms of verification. The analysis on how assurances regarding the proxy model's performance can be extended to the target model is conducted in Section~\ref{sec:error_analysis}.

\subsection{Output Specification: Polytope Representation}\label{sec:output_spec}
For the coordinate $\bbv_k$ of the $k$-th keypoint, where $\bbv_k$ represents its 2D coordinate, let $\bar{\bbv}_k$ denote its \textit{averaged} ground-truth coordinate following the application of the average pooling and argmax layers. To ensure the classification result is consistent across perturbations, the output specification necessitates that the value at $\bar{v}_k$ exceeds the values at all other entries. This requirement can be encapsulated by a half-space-based polytope, represented as $\bbA_k \bby_k \leq \bbb_k$ with
{
\begin{align}
  \bbA_k = 
\begin{bmatrix}
    1 & 0  & \dots & -1 & \dots & 0 \\
    0 & 1  & \dots & -1 & \dots & 0 \\
    \vdots & \vdots & \ddots & \vdots & \vdots & \vdots \\
    0 & 0 & \dots & -1 & \dots & 0 \\
      \vdots & \vdots & \vdots & \vdots & \ddots & \vdots \\
          0 & 0 & \dots & -1 & \dots & 1 \\
\end{bmatrix}
, \quad
\bby_k = 
\begin{bmatrix}
    y_{1} \\
    y_{2} \\
    \vdots \\  
    y_{\bar{v}_k} \\
    \vdots \\
    y_{n}
\end{bmatrix}
, \quad
\bbb_k = 
\begin{bmatrix}
    0 \\
    0 \\
    \vdots \\
    M \\
    \vdots\\
    0
\end{bmatrix}, \label{eq:output}
\end{align}}%
where the $\bar{v}_k$-th column of $\bbA_k$ is populated with (-1)'s, and all diagonal elements set to 1 with the exception of the element at $(\bar{v}_k, \bar{v}_k)$, the components of $\bbb_k$ are set to zero except for the $\bar{v}_k$-th element, which is assigned a significantly large integer $M$. Consequently, the specification for all keypoints is denoted by $\bbA \bbx \leq \bbb$ with
\begin{align}\label{eq:output}
    \mathbf{A} = 
\begin{bmatrix}
    \mathbf{A}_1 & 0 & 0 & \dots & 0 \\
    0 & \mathbf{A}_2 & 0 & \dots & 0 \\
    0 & 0 & \mathbf{A}_3 & \dots & 0 \\
    \vdots & \vdots & \vdots & \ddots & \vdots \\
    0 & 0 & 0 & \dots & \mathbf{A}_K
\end{bmatrix}
, \quad
\bby = 
\begin{bmatrix}
    \bby_1 \\
    \bby_2 \\
    \bby_3 \\
    \vdots \\
    \bby_K
\end{bmatrix}
, \quad
\bbb = 
\begin{bmatrix}
    \bbb_1 \\
    \bbb_2 \\
    \bbb_3 \\
    \vdots \\
    \bbb_K
\end{bmatrix}.
\end{align}
The construction of Eq.~\eqref{eq:output} presupposes verifying all keypoints. However, should there be a keypoint that does not require verification, potentially due to being an outlier, the matrices associated with it are then adjusted accordingly: $ \bbA_k = \mathbf{I}$ and $\bbb_k = [M, \ldots, M]^T$.



\section{Error Propagation and Determination of Pooling Parameters}\label{sec:error_prop_pooling}
To calculate the parameters for average pooling, the initial step involves mapping the error thresholds $\bm{\epsilon}_r$ and $\bm{\epsilon}_t$ from the pose estimation onto the keypoints' error thresholds that may be correlated to each other, which is accomplished by employing sensitivity analysis methods.
Subsequently, we formulate an optimal error threshold allocation problem, aiming to assign independent error thresholds to each keypoint. This step is crucial for determining the allowable errors for each keypoint individually, ensuring that the overall pose estimation adheres to predefined error limits. The idea is graphically depicted in Fig.~\ref{fig:overview}.


\subsection{Error Propagation via Sensitivity Analysis}
Sensitivity analysis for nonlinear optimization explores the impact of first-order changes in the optimization parameters on the locally optimal solution~\cite{castillo2008sensitivity}. It involves analyzing an unconstrained parameterized nonlinear optimization problem expressed as $\min_{\bbx} F(\bbx; \bba)$, with $\bba$ representing the parameter vector. Here, let $z$ represent the optimal objective value given the parameter vector $\bba$, i.e.,
$
    z (\bba) = \min_{\bbx} F(\bbx; \bba). 
$
Sensitivity analysis~\cite{castillo2008sensitivity} connects the first-order derivatives $\partial \bbx$ and $\delta z$ with $\partial \bba$:
\begin{align}\label{eq:general_sensitivity}
   \begin{bmatrix}
       -F_{\bba} \\ -F_{\bbx\bba}  
    \end{bmatrix} \partial \bba  
    = \begin{bmatrix}
        F_{\bbx} & -\mathbf{1} \\ F_{\bbx\bbx} & \mathbf{0}
    \end{bmatrix}
    \begin{bmatrix}
        \partial \bbx \\ \partial z
    \end{bmatrix}.
\end{align}
In the context of pose estimation, the keypoint coordinates $\bbv$ are considered as parameters and the pose to be estimated represents the decision variables, the notation for partial derivatives $\partial$ is substituted with the notation for deivation $\delta$. Consequently, Eq.~\eqref{eq:general_sensitivity} is transformed into
\begin{align}
   \begin{bmatrix}
       -G_{\bbv} \\ -G_{\bbxi\bbv}  
    \end{bmatrix} \delta \bbv
    = 
    \begin{bmatrix}
        G_{\bbxi} & -\mathbf{1} \\ G_{\bbxi\bbxi} & \mathbf{0}
    \end{bmatrix}
    \begin{bmatrix}
        \delta \bbxi \\ \delta z
    \end{bmatrix},
\end{align}
where $\bbxi\in \mbR^{6} = [\bbr, \bbt]$ represents the 6D pose vector, with the 3D rotation matrix $\bbR$ expressed in the axis-angle or Euler angles representations form $\bbr$, and $G$ corresponds to the NLS optimization objective outlined in Eq.~\eqref{eq:nls}. Note that $\delta \bbv \in \mbR^{2K}$. The following notation is introduced
\begin{align}
 \bbM_{\bbv\bbxi} :=  \begin{bmatrix}
       -G_{\bbv} \\ -G_{\bbxi\bbv}  
    \end{bmatrix} \in \mbR^{7\times 2K}, \; 
   \bbM_{\bbxi} :=  \begin{bmatrix}  
        G_{\bbxi} & -\bm{1} \\ G_{\bbxi\bbxi} & \bm{0} 
    \end{bmatrix} \in \mbR^{7\times7}.
\end{align}
If $\bbM_{\bbxi}$ is invertible, then it follows that
$    \bbM_{\bbxi}^{-1} \bbM_{\bbv\bbxi} \delta \bbv =  [   \delta \bbxi, \delta z]^T$.
By extracting the rows in the matrix $\bbM_{\bbxi}^{-1} \bbM_{\bbv\bbxi}$ that correspond to $\delta \bbxi$ and naming the resulting matrix $\tilde{\bbM}_{\bbv\bbxi}$, we obtain the linear mapping 
$\Tilde{\bbM}_{\bbv\bbxi} \delta \bbv = \delta \bbxi$.
This leads to the establishment of a linear relationship between  2D keypoints errors and 6D pose errors. Upon specification of the keypoints $\bbv$ and the transform $\bm{\xi}$, the matrix $\Tilde{\bbM}_{\bbv\bbxi}$ becomes determined. 

Let $\Xi$ represent the set of tolerable errors, indicating the range within which deviations of the perturbed transform from the nominal transform are acceptable, meaning $\delta \bbxi \in \Xi$. {\color{\modifycolor}Let $\delta \ccalV$ denote the derived tolerable keypoint  errors that has the following property:
 $\delta \ccalV =  \left\{\delta \bbv \,|\,  \Tilde{\bbM}_{\bbv\bbxi} \delta \bbv   =  \delta \bbxi  \in \Xi \right\}$.
Note that $\delta \ccalV$ is an approximation of actual tolerable keypoint errors that lead to tolerable pose errors $\delta \bm{\xi}$. As the pose error bounds decrease, the accuracy of the PnP linearization improves, narrowing this disparity.} In what follows, we assume that $\Xi$ is defined as a polytope represented by linear inequalities, specifically $\Xi= \{\delta \bbxi \,|\, \bbP_{\bm{\xi}} \delta \bbxi \leq \bbb_{\bbxi}\}$. This inequality can be established based on a user-defined thresholds $\bm{\epsilon}_r$ and $\bm{\epsilon}_t$ as outlined in Problem~\ref{prob:certification}. By the linear mapping, the set $\delta\ccalV$ can also be characterized as a polytope:
    $\delta\ccalV = \left\{\delta \bbv \,|\, \bbP_{\bm{\xi}} \Tilde{\bbM}_{\bbv\bbxi} \delta \bbv \leq \bbb_{\bbxi}\right\} = \left\{ \delta \bbv \,|\, \bbP_{\bbv}  \delta \bbv \leq \bbb_{\bbv} \right\}$,
where $\bbP_{\bbv} =  \bbP_{\bbxi} \Tilde{\bbM}_{\bbv\bbxi}$ and $\bbb_{\bbv} = \bbb_{\bbxi}$. {\color{\modifycolor} While larger set  $\delta \ccalV$ is desirable, it is an approximation of actual tolerable errors due to the linearization of sensitivity analysis. To mitigate the risk of over-approximation, we  introduce a scaling factor to reduce the size of  set $\delta \ccalV$ in the next section.}

\subsection{Optimal Error Threshold Allocation}\label{sec:opt_error_allocation}
The tolerable errors of keypoints are interrelated within the derived polytope set $\delta\ccalV$. {\color{\modifycolor}In this section, we allocate the tolerable errors across each keypoint to ensure their tolerances are independent, as current verification tools lack the capability to verify dependencies between  keypoints.} The allocated error threshold  aims to be maximized, provided that the pose deviation remains tolerable.

For a given error vector $\delta \bbv$, we can uniquely define an axis-aligned, axis-symmetric hyper-rectangle, denoted as $\ccalH\ccalR(\delta \bbv) \subset \mbR^{2K}$, centered at the origin with each element $\delta \bbv_i$ representing the half-length of its sides. Axis symmetry indicates that the error thresholds are identical in opposite directions. We establish the following problem for optimal error threshold allocation:
{\allowdisplaybreaks
\begin{align}
  \max_{\delta \bbv \in \bbR_+^{2K}}  \quad & w_1 \prod_{k=1}^{2K} \delta \bbv_k + w_2 \Delta \label{eq:obj}\\
 \text{s.t.}\quad   & \ccalH\ccalR(\delta \bbv)  \subseteq \delta\ccalV, \label{eq:inclusion} \\
   \quad & \delta \bbv_k\geq \Delta, \quad  k = 1, \ldots, 2K. \label{eq:lower_bound}
\end{align}
}%
Constraint~\eqref{eq:inclusion} ensures that the hyper-rectangle is encompassed within the set of tolerable errors $\delta\ccalV$, while constraint~\eqref{eq:lower_bound} sets a minimum threshold for the side lengths. The goal defined in objective~\eqref{eq:obj} is to optimize a weighted sum of two objectives: the first seeks to identify the largest possible axis-aligned, axis-symmetric hyper-rectangle within a convex polytope, and the second strives to extend the minimal side length as much as feasible.

To address the optimal error threshold allocation problem, we translate constraint~\eqref{eq:inclusion} into a series of linear inequalities, drawing from the method in~\cite{behroozi2019largest} for identifying the largest axis-aligned hyper-rectangle inscribed within a convex polytope. Considering a hyper-rectangle $\ccalB = \{ \bbx \in \mbR^d \mid \bbl \leq \bbx \leq \bbu \}$ and a convex polytope $\ccalC = \{ \bbx \mid \bbP \bbx \leq \bbb \}$,~\cite{behroozi2019largest} demonstrates that $\ccalB \subseteq \ccalC$ if and only if
\begin{align}\label{eq:inclusion_ineq}
     \sum_{j=1}^{d} (\bbP_{ij}^+ \bbu_j - \bbP_{ij}^- \bbl_j) \leq \bbb_i, \quad i = 1, \ldots, n.
\end{align}
where $\bbP_{ij}^+ = \max\{\bbP_{ij},0\}$ and $\bbP_{ij}^- = \max\{-\bbP_{ij},0\}$. Given the axis symmetry in our scenario, it follows that $\bbu_j = -\bbl_j$. Substituting $\bbu_j$ with $\delta\bbv_j$, $\bbP$ with $\bbP_{\bbv}$, and $\bbb$ with $\bbb_\bbv$, we can reformulate Ineq.~\eqref{eq:inclusion_ineq} as
\begin{align}
    \sum_{j=1}^{2K} \delta\bbv_j (\bbP_{\bbv ij}^+  + \bbP_{\bbv ij}^-) \leq \bbb_{\bbv i}, \quad i = 1, \ldots, 6. \quad
\end{align}
Given that $\left\lvert \bbP_{\bbv ij}\right\rvert = \bbP_{\bbv ij}^+ + \bbP_{\bbv ij}^-$ where $\vert \cdot \vert$ returns element-wise absolute values, it follows that
\begin{align}\label{eq:inclusion_ineq_abs}
    \sum_{j=1}^{2K} \delta\bbv_j \left\lvert \bbP_{\bbv ij}\right\rvert \leq \bbb_{\bbv i}, \quad i = 1, \ldots, 6. \quad
\end{align}
In matrix notation, Ineq.~\eqref{eq:inclusion_ineq_abs} is expressed as
$
     |\bbP_\bbv| \delta\bbv \leq \bbb_{\bbv}. 
$
Ultimately, the problem of optimal error threshold allocation, encompassing objectives~\eqref{eq:obj} through constraints~\eqref{eq:lower_bound}, is equivalent to
{\allowdisplaybreaks
\begin{equation}\label{eq:formulation}
\begin{split}
  \max_{\delta \bbv \in \mbZ_+^{2K}}  \quad & w_1 \prod_{k=1}^{2K} \delta \bbv_k + w_2 \Delta \\
 \text{s.t.}\quad   &   \kappa\, |\bbP_\bbv| \delta\bbv \leq \bbb_{\bbv}, \\
   \quad & \delta \bbv_k\geq \Delta, \quad k = 1, \ldots, 2K,
\end{split}
\end{equation}
}%
where $\mbZ_+$ represents the set of non-negative integers, and $\kappa\in \mbR_+$ is the scaling factor that controls the actual size of polytope. The effect of the scaling factor $\kappa$ on the allocation is discussed in Section~\ref{sec:prob_property}. The solution, denoted by $\delta\bbv^*$, returns the maximum tolerable errors for all keypoints while ensuring that the pose errors remains within the error thresholds. This solution will determine the parameters for average pooling in Section~\ref{sec:model}.

\subsection{Determination of Average Pooling Parameters}

{\color{\modifycolor}A widely accepted practice to train the keypoint detection model is to regress the heatmaps centered on the ground truth keypoint. Building on this approach, we propose the following assumption and support it with evaluation results detailed in Section~\ref{sec:asmp}.}

\begin{assumption}\label{asmp:gaussian}
  {\color{\modifycolor} The softmax layer in the target model $\bbF_{\text{target}}$ produces a distribution with a peak at its center and exhibits symmetry along axes that intersect this peak.}
\end{assumption}

Upon determining the independent error threshold for each keypoint, we proceed to define the parameters for average pooling in the proxy model $\bbF_{\text{proxy}}$. Focusing on a particular keypoint with coordinates $\bbv = (\bbv_h, \bbv_v)$, and considering the error thresholds $\delta\bbv^*_h$ and $\delta\bbv^*_v$ allocated for horizontal and vertical directions respectively, we define the set 
\begin{align}
    \ccalV_{\delta} = \{ (h, v) \mid | h - \bbv_h| \leq \delta \bbv^*_h \land | v - \bbv_v| \leq \delta \bbv^*_v \}.
\end{align}
Only pixels within $\ccalV_\delta$ are considered acceptable potential predictions for the keypoint. The average pooling parameters for this keypoint are selected to ensure every pixel in $\ccalV_{\delta}$ is contained by the same pooling patch. This guarantees that all these pixels are pooled into the same after-pooling pixel as the keypoint, thereby belonging to the same class.

We begin by examining how predictions are generated in the proxy model $\md{proxy}$. Based on Assumption~\ref{asmp:gaussian}, the unnormalized heatmap returned by the backbone network $\bbF_{\text{b}}$ exhibits symmetry along axes passing through this peak, as softmax operation will not change the overall shape of the unnormalized heatmap. To infer the keypoint prediction, we concentrate on the average pooling patch that predominantly overlaps with the vicinity of the peak. Since the average pooling patches also exhibit axial symmetry, patches closer to the peak yield larger averaged pixel values. Consequently, the pooling patch with its center closest to the peak results in the predicted keypoint. This insight enables us to transform the determination of the location of the maximal pixel value after pooling into identifying the closest pooling patch to the peak of the unnormalized heatmap under Assumption~\ref{asmp:gaussian}. Next, we compute the average pooling parameters for each keypoint to ensure that: (i) the ground-truth keypoint aligns with the center of an average pooling patch, and (ii) {\color{\modifycolor}all pixels in proximity to this pooling patch fall within the allocated error thresholds from the ground truth keypoint, making them suitable to act as the peak of the unnormalized heatmap. Let $\bbs = (s_h, s_v)$ denote the stride, $\bbk = (k_h, k_v)$ the kernel size, and $\bbp = (p_h, p_v)$ the patch size.}

\begin{figure}[!t]
  \begin{minipage}[b]{0.6\textwidth}
    \centering
   \subfigure[]{
      \label{fig:pooling}
      \includegraphics[ width=\linewidth]{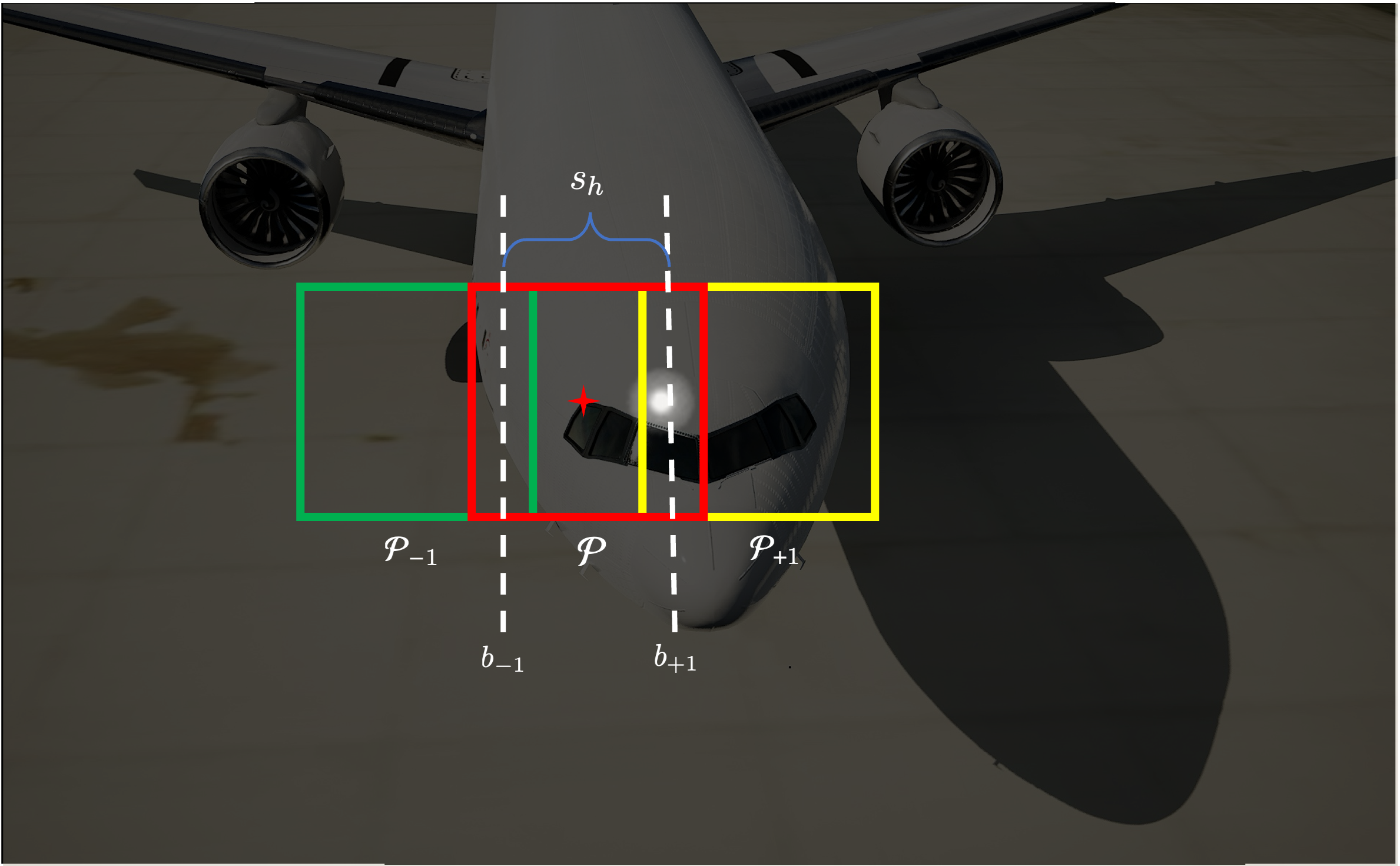}}
  \end{minipage}
  \hfill 
  \begin{minipage}[b]{0.32\textwidth}
    \centering
    \subfigure[]{
    \includegraphics[trim=10 50 30 40, clip, width=\textwidth]{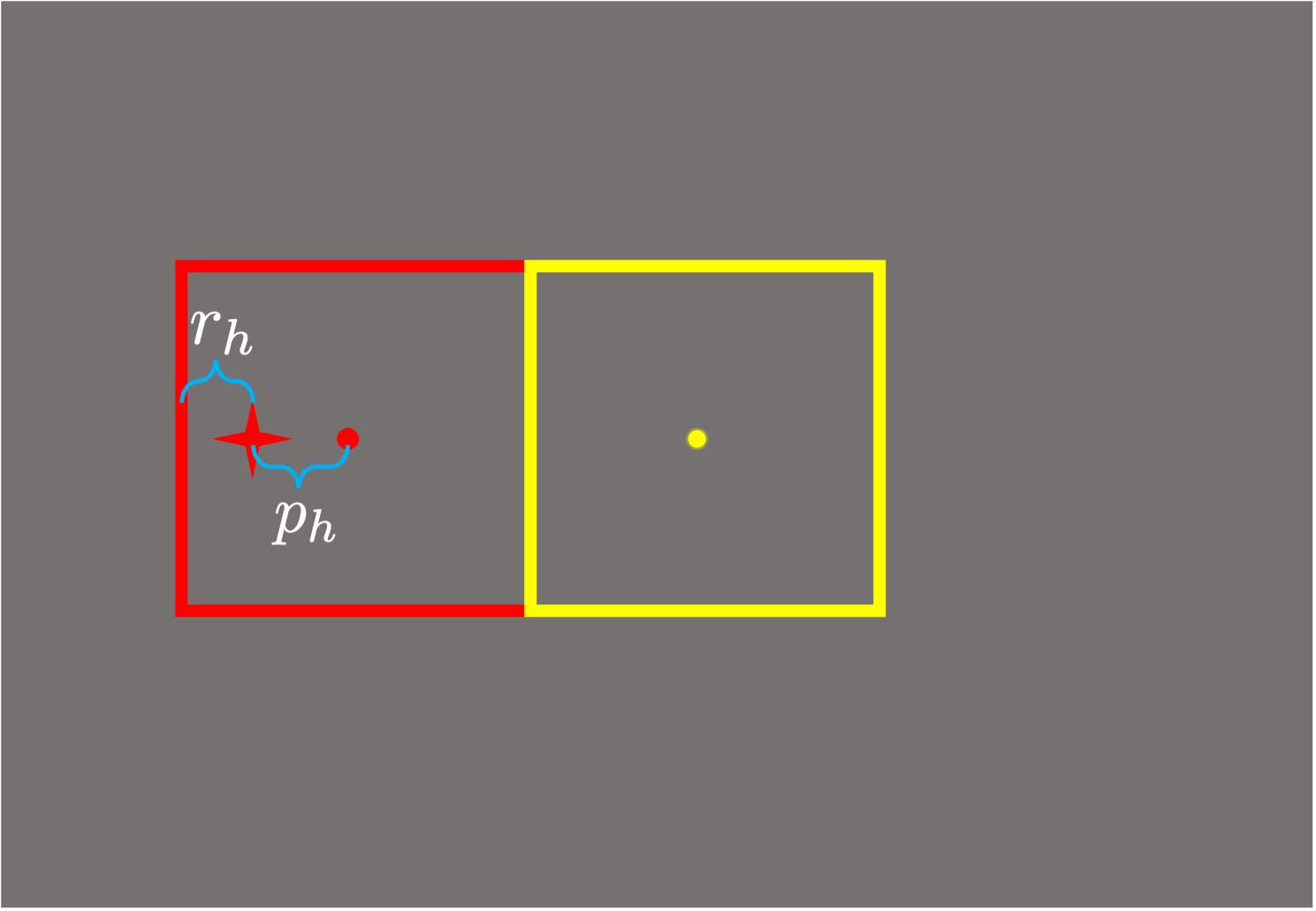}
    \label{fig:padding_a}
    }
    \subfigure[]{
    \includegraphics[trim=10 50 30 40, clip, width=\textwidth]{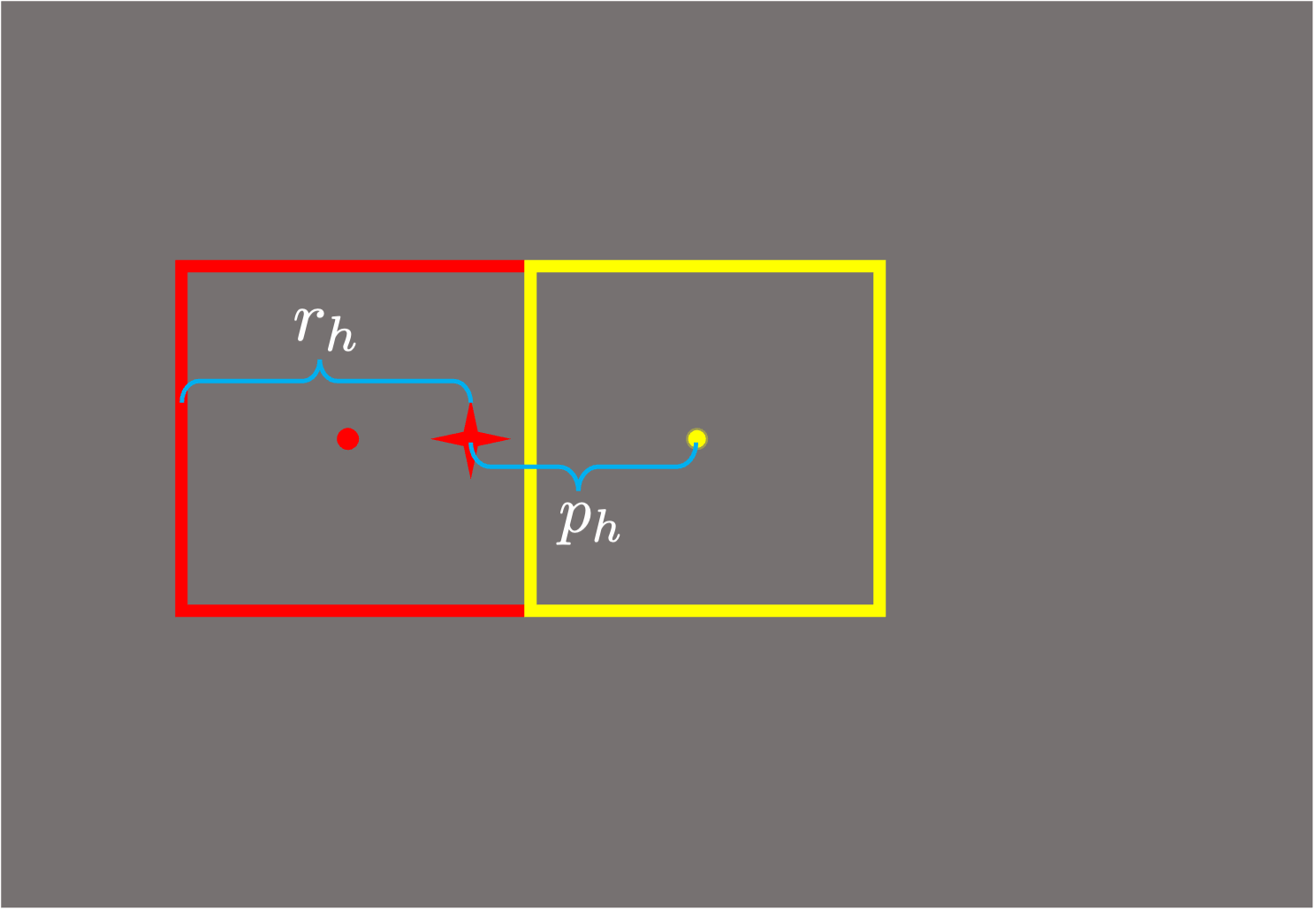}
    \label{fig:padding_b}
    }
  \end{minipage}
  \vspace{-10pt}
   \caption{{\color{\modifycolor} (a) Graphical depiction of determining the stride parameter. The unnormalized heatmap is overlaid on the airplane.  The most saturated area in the heatmap, located near the center, indicates its peak.  Consecutive average pooling patches are in different colors, denoted as $\ccalP_{-1}, \ccalP$, and $\ccalP_{+1}$, with dots indicating their centers. The {\color{red}red} star \raisebox{-3pt}{\includegraphics[height=0.9\baselineskip]{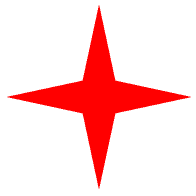}} denotes the ground-truth keypoint, which aligns with the center of the average pooling patch $\ccalP$.
   The dashed vertical lines, $b_{-1}$ and $b_{+1}$, represent perpendicular bisectors between the centers of adjacent patches, with a distance equal to the stride $s_h$.  Consequently, the {\color{red}red} pooling patch corresponds to the predicted keypoint. (b) and (c): Determination of the padding parameter.}}
\end{figure}

\paragraph{Determination of stride and kernel parameters.} 
The initial step involves determining the stride and kernel parameters, denoted as $s_h$ and $k_h$, respectively, as in Fig.~\ref{fig:pooling}. We focus on the horizontal axis (similar logic applies to the independent vertical axis). Let $\ccalP$ represent the pooling patch centered at a keypoint, with $\ccalP_{-1}$ and $\ccalP_{+1}$ denoting the adjacent patches to the left and right, respectively. Additionally, let $b_{-1}$ and $b_{+1}$ denote the perpendicular bisectors between these patches. The distance between these two perpendicular bisectors corresponds to the horizontal stride. To position the ground-truth keypoint at the center, an odd stride is required. As long as the peak of the unnormalized heatmap lies between $b_{-1}$ and $b_{+1}$, the pooling patch $\ccalP$ is the closest to the peak, resulting in the predicted keypoint. Let $\delta \bbv^*_h$ denote the allocated error threshold. To ensure that peaks falling between $b_{-1}$ and $b_{+1}$ remain permissible with respect to the error threshold, the threshold should be no less than half the stride distance from the ground-truth keypoint on either side, i.e.,
$
     \delta\bbv^*_h \geq \frac{1}{2} (s_h - 1),
$
which leads to the inequality $s_h \leq 2\delta\bbv^*_h + 1$. Choosing $s_h =  2 \delta \bbv^*_h + 1$ minimizes the post-pooling size. It's worth noting that there are no constraints on the kernel parameter. For simplicity, we set the kernel $k_h$ to be equal to the stride $s_h$.

\paragraph{Determination of the padding parameter.}
The next step involves determining the padding parameter to center a pooling patch at the ground-truth keypoint. Padding can be added from the left (horizontal) or from the top (vertical) to achieve this. Considering the horizontal axis where the kernel is equal to the stride, the idea is to identify the closest pooling patch to the keypoint from its right side and then shift it left by the distance between the keypoint and its center, as illustrated in Fig.~\ref{fig:padding_a} and~\ref{fig:padding_b}. Assuming $k_h = s_h$, let $r_h$ be the number of pixels the keypoint is from the left side of the closest pooling patch, {\color{\modifycolor}calculated as $r_h \equiv \bbv_h \pmod{k_h}$}. (i) If $r_h \leq \frac{\text{k}_h + 1}{2}$, the closest patch center to the right of the keypoint, belongs to the pooling patch containing the keypoint. (ii) Otherwise, the closest patch center to the right comes from the pooling patch immediately to the right of the one containing the keypoint. The horizontal padding parameter $p_h$ is determined as $\frac{k_h+1}{2} - r_h$ if $ r_h \leq \frac{k_h+1}{2}$, otherwise, as $\frac{3k_h+1}{2} - r_h$.
{\color{\modifycolor}Note that the shift is not unique because multiples of the stride can be added to $p_h$.}


\section{Theoretical Analysis}\label{sec:error_analysis}
{\color{\modifycolor}With the method we introduced (to reformulate the problem into a verifiable form using off-the-shelf verification tools), it is important to understand whether false analysis results will be introduced. To answer that question, we perform the following analysis.}

\subsection{Properties of the Proxy Model}
We examine two target models. The models ending with a DSNT layer and an argmax layer are represented by $\md{dsnt}$ and $\md{argm}$, respectively. The term $\md{target}$ is used to represent either target model. The notation $\md{}^k$ is employed to identify the $k$-th predicted keypoint. Recall that $\bar{\bbv}_k$ denote the averaged ground-truth coordinate of the $k$-th keypoint after the application of average pooing and argmax layers. Before providing theoretical results, we establish formal definitions for both the soundness and completeness related to the proxy model, which are visually represented in Fig.~\ref{fig:theory}.

\begin{definition}[Soundness]\label{def:soundness}
    A proxy model is deemed {\it sound} if, for cases where the proxy model's predicted keypoint aligns with the averaged ground-truth keypoint $\bar{\bbv}_k$, the deviation of the target model's prediction from the ground truth $\bbv_k$ does not exceed the error threshold $\delta\bbv_k^*$. Specifically, 
    \begin{align}
        \md{proxy}^k(\bbX) = \bar{\bbv}_k \Rightarrow \left\vert \md{target}^{k}(\bbX) - \bbv_k \right\vert \leq \delta\bbv_k^*, \quad \text{for}\; k = 1, \ldots, K.
    \end{align}
     Here, $\vert \cdot \vert$ indicates the computation of element-wise absolute differences.
\end{definition}
\begin{definition}[Completeness]\label{def:completeness}
    A proxy model is deemed {\it complete} if the difference between the target model's prediction and the ground truth  $\bbv_k$ is confined within the error threshold $\delta\bbv^*_k$, then the predicted keypoint by the proxy model matches the averaged ground truth $\bar{\bbv}_k$. That is, 
    \begin{align}
        \left\vert \md{target}^{k}(\bbX) - \bbv_k \right\vert \leq \delta\bbv_k^* \Rightarrow \md{proxy}^k(\bbX) = \bar{\bbv}_k, \quad \text{for}\; k = 1,\ldots, K.
    \end{align}
\end{definition}

\begin{theorem}\label{thm:soundness}
If Assumption~\ref{asmp:gaussian} holds, the proxy model is sound, that is,
\begin{align}
   \md{proxy}^{k}\left(\bbX\right) = \bar{\bbv}_k & \Rightarrow
   \left\vert \md{dsnt}^{k}\left(\bbX\right) - \gti  \right\vert \leq \delta\bbv_k^*,   \quad \text{for}\;  k = 1, \ldots, K. \label{eq:dsnt_imp} \\
 \md{proxy}^{k}\left(\bbX\right) = \bar{\bbv}_k & \Rightarrow   
 \left\vert \md{argm}^{k}\left(\bbX\right) - \gti \right\vert \leq \delta\bbv_k^*,   \quad \text{for}\;  k = 1, \ldots, K. \label{eq:argm_imp}
\end{align}
\end{theorem}
To prove Theorem~\ref{thm:soundness}, we first prove the equivalence of two target models.
    \begin{proposition}\label{prop:equiv}
      If Assumption~\ref{asmp:gaussian} holds, the two target models produce the same predictions, i.e.,
    $\md{dsnt}^k(\bbX) =  \md{argm}^k(\bbX),    \quad \text{for}\;   k = 1, \ldots, K$.
    \end{proposition}
    
\begin{figure}[!t]
    \centering
    \includegraphics[width=0.9\linewidth]{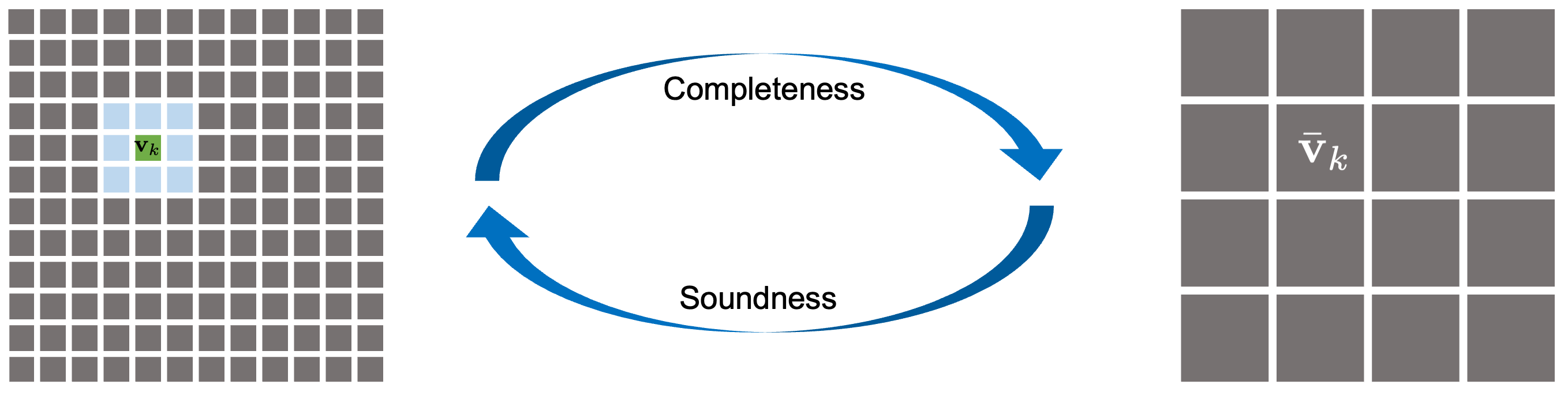}
    \caption{Illustration of soundness and completeness of the proxy model: The left heatmap, sized $12\times12$ and generated by the target model, is transformed into a $4\times4$ heatmap by the proxy model, utilizing pooling patches with both kernel and stride set to 3. Both the averaged ground-truth keypoint $\bar{\bbv}_k$ and the ground-truth keypoint $\bbv_k$ are marked in {\color{green}green}. The area highlighted in {\color{blue}blue} encompasses pixels located within a $\delta \bbv_k^*$ distance from the ground truth.}
    \label{fig:theory}
\end{figure}

\begin{proof}

{\color{\modifycolor} Before delving into the proof, we first introduce how DSNT extracts a keypoint from a normalized heatmap. Let $\bbH$ represent a normalized heatmap of dimensions $m \times n$, which is generated by a softmax layer. Define $\bbC$ as an $m \times n$ matrix where each entry $\bbC_{ij}$ is calculated by $\bbC_{ij} = \frac{2j-(n+1)}{n}$. Consequently, each column of $\bbC$ contains identical values. These values fall within the range of $(-1, 1)$, with the center value at 0. To determine the column coordinate $\bbv_h$ of the keypoint, DSNT computes the sum of the element-wise product between $\bbH$ and $\bbC$, expressed as $\bbv_h = \sum_{i, j} \bbH_{ij} \bbC_{ij}$. The row coordinate is obtained in a similar way.}

Without loss of generality, due to independence of horizontal and vertical axes, we will show for the horizontal direction. Considering a symmetric heatmap with dimensions $m \times n$, we initially position the center of the distribution at the $\frac{m+1}{2}$-th row and the $\frac{n+1}{2}$-th column, effectively placing the distribution's center at the heatmap's midpoint. Upon applying the DSNT process, the resulting keypoint coordinate is (0, 0). This outcome is attributed to the distribution's symmetry along both the horizontal and vertical axes that intersect at its center, with the DSNT operation cancelling out symmetric values relative to the origin. Subsequently, we consider a horizontal displacement of the distribution by $c$ columns, where a positive $c$ indicates a shift to the right. The column coordinate $\bbv_h$, as determined by DSNT following the horizontal shift, is then
{\allowdisplaybreaks
\begin{align}
    \sum_{i, j} \bbH_{ij} \bbC_{ij} \  \overset{\circled{1}}{=} \   \sum_{i, j'} \bbH_{ij'} \left(\bbC_{ij'} +  \frac{2c}{n}\right)\  \overset{\circled{2}}{=} \   \sum_{i, j'} \bbH_{ij'} \bbC_{ij'} + \sum_{i, j'} \bbH_{ij'} \frac{2c}{n} \ \overset{\circled{3}}{=}\   0 + \frac{2c}{n}  \sum_{i, j'} \bbH_{ij'}  \ \overset{\circled{4}}{=} \  \frac{2c}{n} \label{eq:sum_distribution}
\end{align}
}%
In~\circled{1}, the pixel located at $(i, j')$ is the corresponding pixel at $(i, j)$ prior to the horizontal shift, with both pixels having identical values, denoted as $\bbH_{ij} = \bbH_{ij'}$. The value $\frac{2c}{n}$ represents the  difference across every $c$ columns within the  matrix $\bbC$.
For~\circled{3}, the expression $\sum_{i, j'} \bbH_{ij'} \bbC_{ij'}$ equals zero, reflecting the distribution's initial centering at the heatmap's midpoint. In~\circled{4}, the probability over distribution sums up to 1. The calculation of the DSNT column coordinate as $\frac{2c}{n} \in (-1, 1)$ aligns it with the heatmap's $(\frac{n+1}{2}+c)$-th column. On the other hand, the argmax operation in $\md{argm}^k$ returns the  distribution's center, located at the $(\frac{n+1}{2}+c)$-th column following the shift. Consequently, we establish that $\md{dsnt}^{k}(\bbX) = \md{argm}^{k}(\bbX)$, thereby completing the proof.
\end{proof}

According to Proposition~\ref{prop:equiv}, it is sufficient to validate Eq.~\eqref{eq:argm_imp} that connects $\md{proxy}$ with $\md{argm}$ to prove Theorem \ref{thm:soundness}. The distinction between the $\md{proxy}$ and $\md{argm}$ models lies solely in their method of handling the unnormalized heatmap, specifically whether they use a softmax or an average pooling layer as depicted in Fig.~\ref{fig:overview}. With this, we are poised to validate Theorem~\ref{thm:soundness}, a conclusion that naturally follows from the way we identify average pooling parameters.
\begin{proof}
{\color{\modifycolor}Considering the $k$-th keypoint $\bbv_k$, if $\md{proxy}^{k}(\bbX) = \bar{\bbv}_k$, it implies that average pooling parameters ensure that the closest pooling patch $\ccalP_{\bbv_k}$ to the ground-truth keypoint $\bbv_k$ is exactly centered at $\bbv_k$.
 The selection of stride parameters ensures that the heatmap's peak is within a $\delta\bbv^*_k$ distance from the center of $\ccalP_{\bbv_k}$, which is also the ground-truth coordinate $\bbv_k$. Given that the softmax layer preserves the peak's position, the prediction $\md{argm}^k$ precisely returns the peak's location, leading to the conclusion that $\md{argm}^{k}(\bbX)$ is within a $\delta\bbv^*_k$ distance from $\bbv_k$, i.e.,
 $\left\vert \md{argm}^{k}(\bbX) - \gti \right\vert \leq \delta\bbv_k^*$.}
\end{proof}

\begin{theorem}
If Assumption~\ref{asmp:gaussian} holds, the proxy model is complete, that is,
\begin{align}
\left\vert \md{dsnt}^{k}\left(\bbX\right) - \gti  \right\vert \leq \delta\bbv_k^* & \Rightarrow \md{proxy}^{k}\left(\bbX\right) = \bar{\bbv}_k,  \quad \text{for}\;  k = 1, \ldots, K. \label{eq:comp_dsnt_imp} \\
\left\vert \md{argm}^{k}\left(\bbX\right) - \gti \right\vert \leq \delta\bbv_k^*  & \Rightarrow   \md{proxy}^{k}\left(\bbX\right) = \bar{\bbv}_k,  \quad \text{for}\;  k = 1, \ldots, K. \label{eq:comp_argm_imp}
\end{align} 
\end{theorem}
\begin{proof}
Based on Proposition~\ref{prop:equiv}, we focus exclusively on the target model $\md{argm}$. The selection of pooling parameters guarantees that every pixel within the $\delta \bbv^*_k$ proximity of the ground-truth keypoint---eligible to be recognized as the predicted keypoint---shares the same closest pooling patch. This closest pooling patch is identical to that of the ground-truth keypoint. As a result, these pixels yield an averaged pixel that is the same as that of the ground-truth keypoints.
\end{proof}

\begin{figure}[!t]
    \centering
    \includegraphics[width=0.7\linewidth]{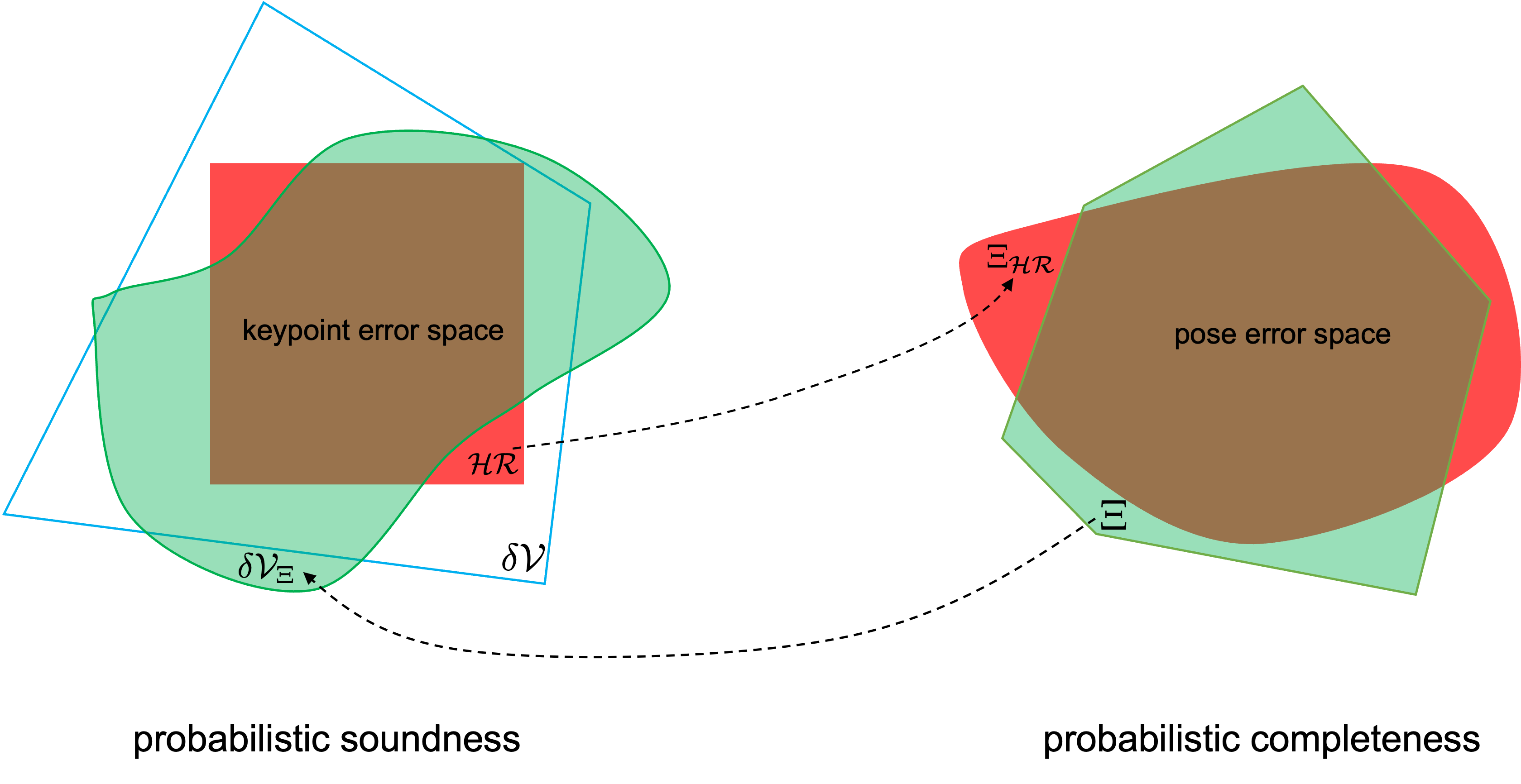}
     \caption{Illustration of the probabilistic soundness and completeness. On the left, the {\color{green} green} area represents the ground-truth tolerable errors on keypoints $\delta \ccalV_{\Xi}$. The {\color{red} red} patches ($\ccalH \ccalR \setminus \delta \ccalV_{\Xi}$) indicate the keypoint errors that result in pose errors exceeding tolerance. The {\color{brown} brown} area ($\ccalH \ccalR \cap \delta \ccalV_{\Xi}$) shows the keypoint errors that lead to tolerable pose errors. Similarly, on the right,  the {\color{brown}brown} area  ($\Xi \cap \Xi_{\ccalH\ccalR}$) represents the pose errors that cause keypoint errors within $\ccalH\ccalR$. {\color{\modifycolor}The dashed black lines indicate the correlations between these two spaces.}}
    \label{fig:prob_prop}
 \end{figure}
 
\subsection{Properties of the Optimal Error Threshold Allocation}
In the previous section, we examined the properties of the proxy model after determining the error threshold for keypoints. This section focuses on the properties of the optimal error threshold allocation. We introduce the concepts of probabilistic soundness and probabilistic completeness to measure the relationship between tolerable errors in pose and keypoints, as depicted in Fig.~\ref{fig:prob_prop}. 

Considering the set of acceptable pose errors $\Xi$, let $\delta \ccalV_{\Xi}$ represent the corresponding set of acceptable keypoint errors. Note that $\delta \ccalV_{\Xi}$ might not be polytopic, due to the non-linear nature of the PnP method, even though $\Xi$ is polytopic. Let $\mu(\cdot)$ denote the measure of a set.
    \begin{definition}[Probabilistic Soundness]
    Probabilistic soundness is defined by the ratio $\frac{\mu(\delta \ccalV_{\Xi} \cap \ccalH\ccalR)}{\mu(\ccalH\ccalR)}$, which is the fraction of keypoint errors in $\ccalH\ccalR$ that result in tolerable pose errors within $\Xi$.
    \end{definition}
    Likewise, considering the set of tolerable errors on keypoints $\ccalH \ccalR$, let $\Xi_{\ccalH\ccalR}$ represent the corresponding set of tolerable errors on poses.
     \begin{definition}[Probabilistic Completeness]
     Probabilistic completeness is defined as the ratio $\frac{\mu(\Xi_{\ccalH\ccalR} \cap \Xi)}{\mu(\Xi)}$, which represents the fraction of permissible pose errors in $\Xi$ that result in errors on keypoints within $\ccalH\ccalR$.
    \end{definition}
Deriving an exact or lower bound for these two metrics is difficult, due to the challenges in making definitive conclusions using linearization-based sensitivity analysis, as the accuracy of local linearizations varies with instances, such as varying poses. In Section~\ref{sec:prob_property}, we perform statistical tests to approximate these two values, offering insights into the practicality of our framework.

\section{Evaluation Results}\label{sec:evaluation}

\subsection{Verified Keypoint Detection Model and Perturbations}
\subsubsection{Verified keypoint detection model}
{\color{\modifycolor}  The verified keypoint detection model, illustrated in Fig.~\ref{fig:verified_model}, is modified from the proxy model $\md{proxy}$. This verified model includes three components:
\paragraph{1. Backbone model}  The dataset contains 7,320 images, each with dimensions of $1920 \times 1200$. 
\begin{itemize}
    \item CNN. The architecture includes 5 convolutional layers and an equal number of deconvolutional layers, designed to handle inputs of size 64$\times$64$\times$3. The model comprises 39 layers and contains around 6.57$\times 10^5$ trainable parameters.
    \item ResNet-18. This model includes 8 residual blocks  and  processes inputs of size 256$\times$256$\times$3. It is composed of 84 layers and possesses approximately 1.2$\times 10^7$ trainable parameters, which is approximately an order of magnitude larger than existing works~\cite{kouvaros2023verification}.
\end{itemize}
}

\begin{figure}[!t]
    \centering
    \includegraphics[width=0.75\linewidth]{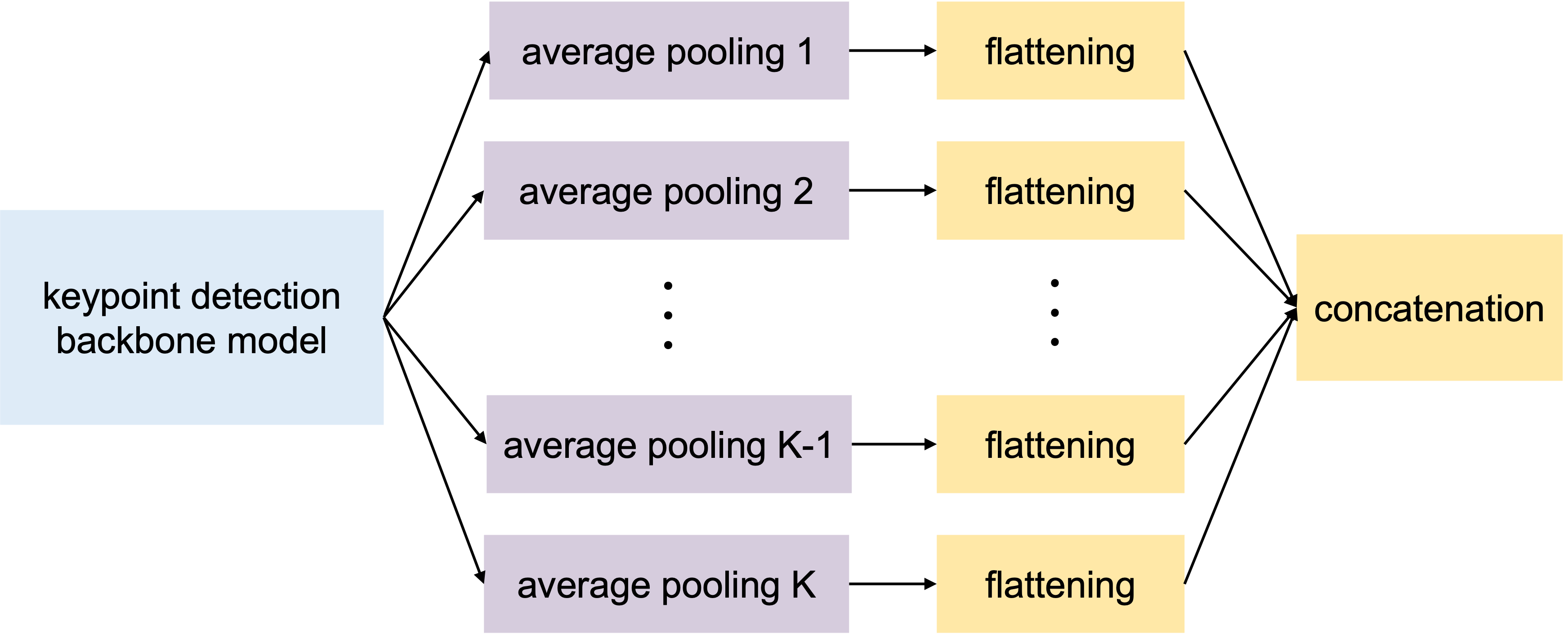}
    \caption{Verified keypoint detection model.}
    \label{fig:verified_model}
\end{figure}

\begin{figure}[!t]
    \centering
     \subfigure[]{
      \label{fig:asset}
      \includegraphics[width=0.35\linewidth]{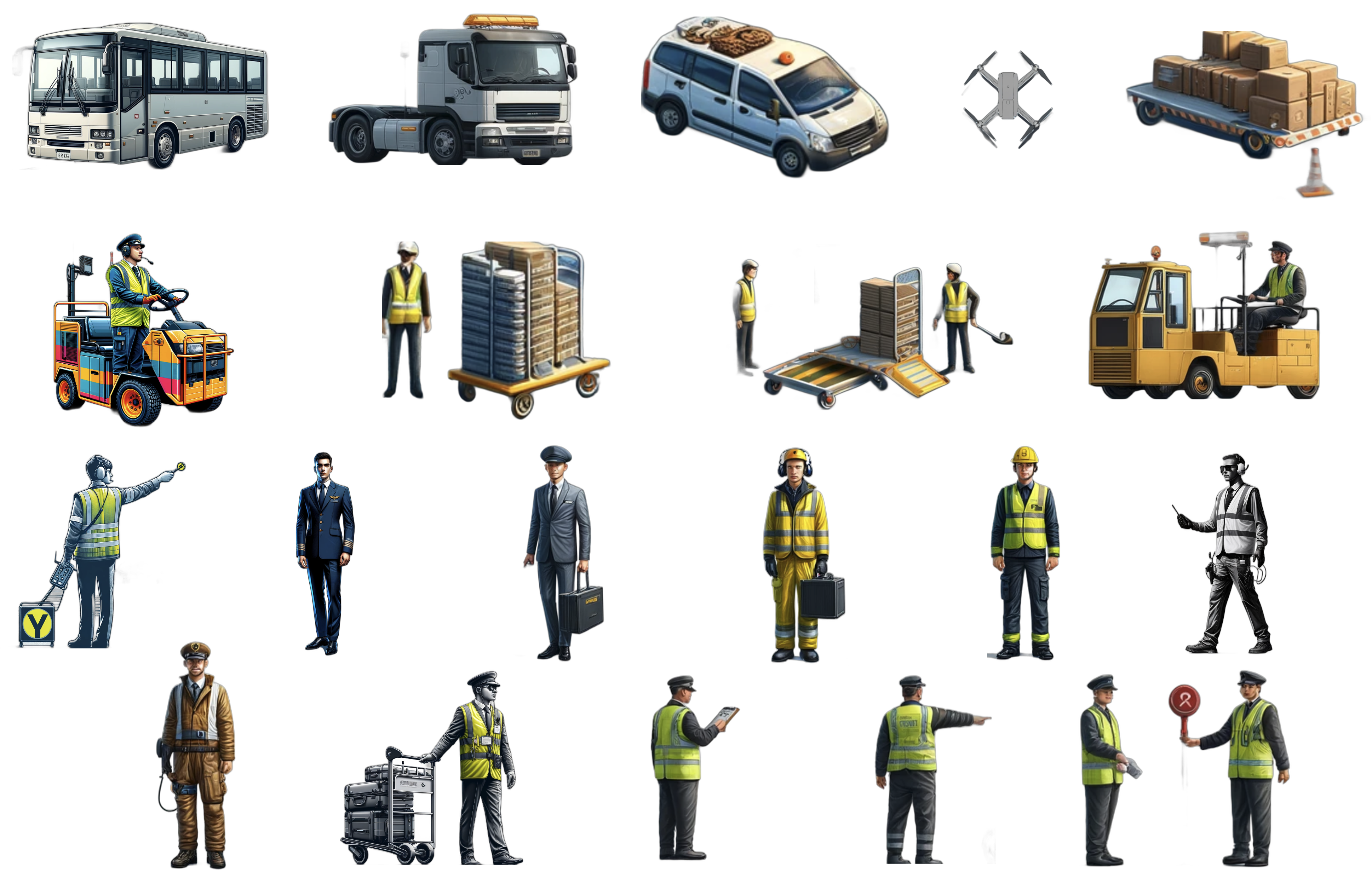}}
     \subfigure[]{
      \label{fig:overlapping_airplanes}
      \includegraphics[width=0.35\linewidth]{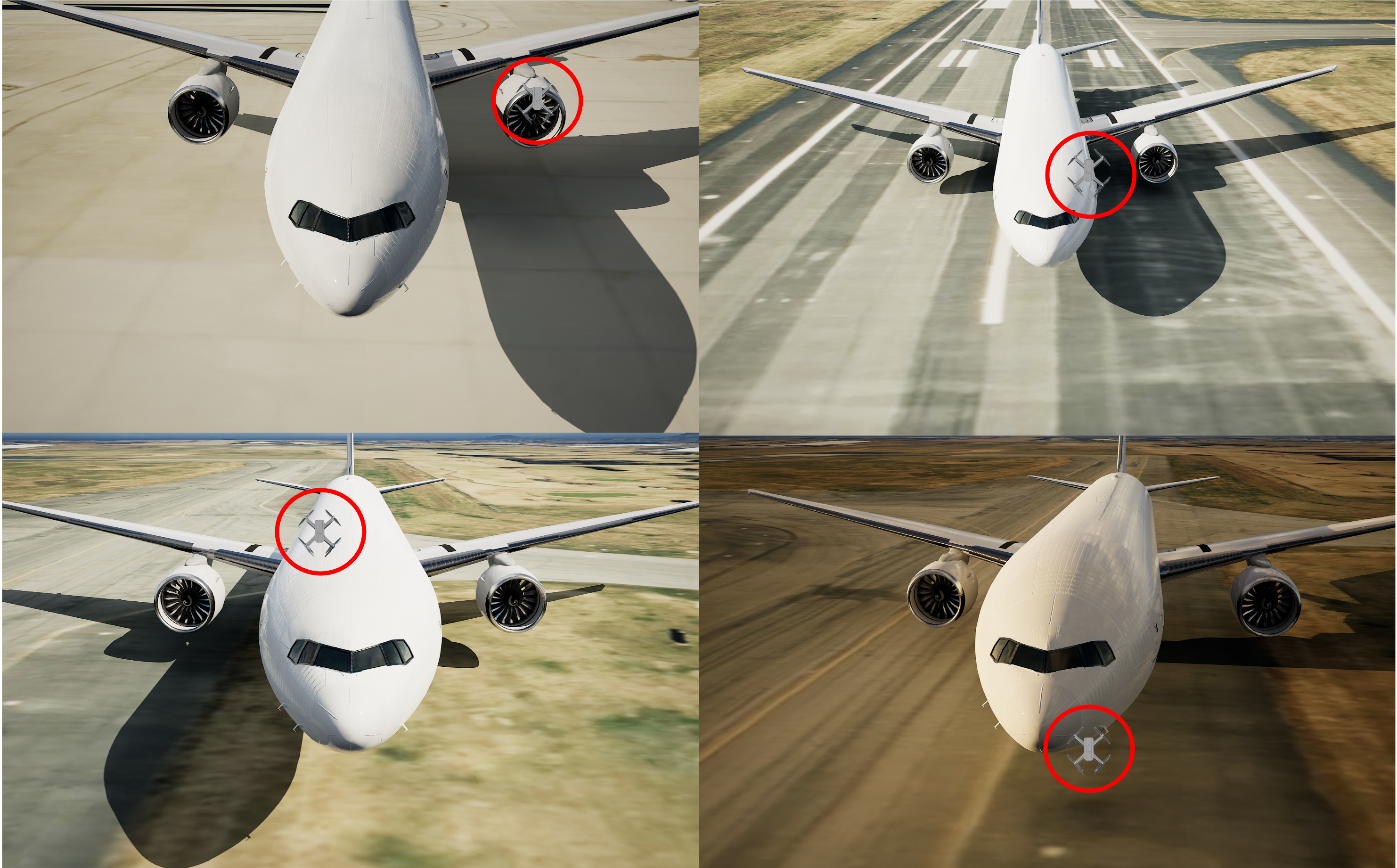}}
        \subfigure[]{
      \label{fig:non-overlapping_airplane}
      \includegraphics[width=0.35\linewidth]{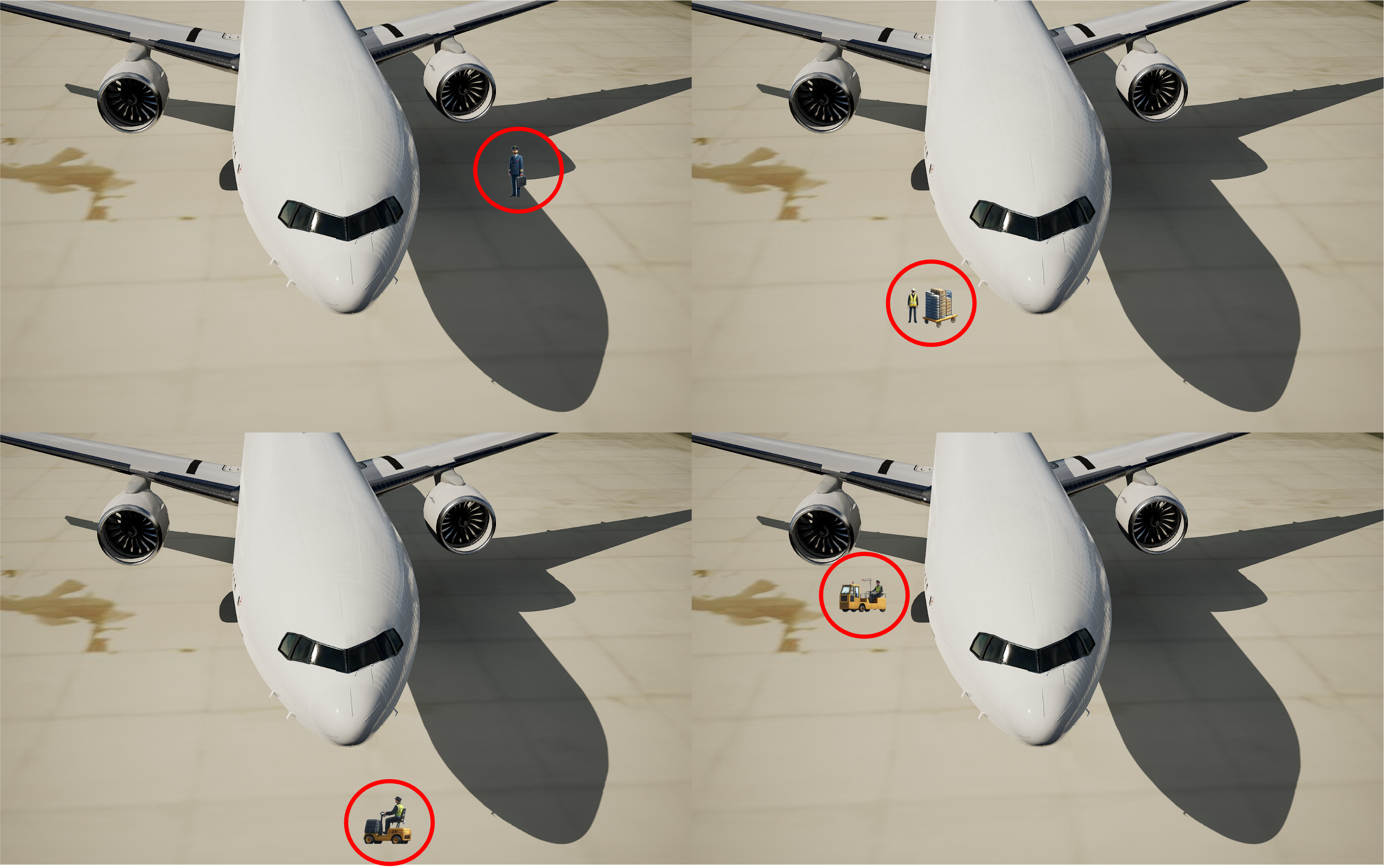}}
        \subfigure[]{
      \label{fig:sampled_airplane}
      \includegraphics[width=0.35\linewidth]{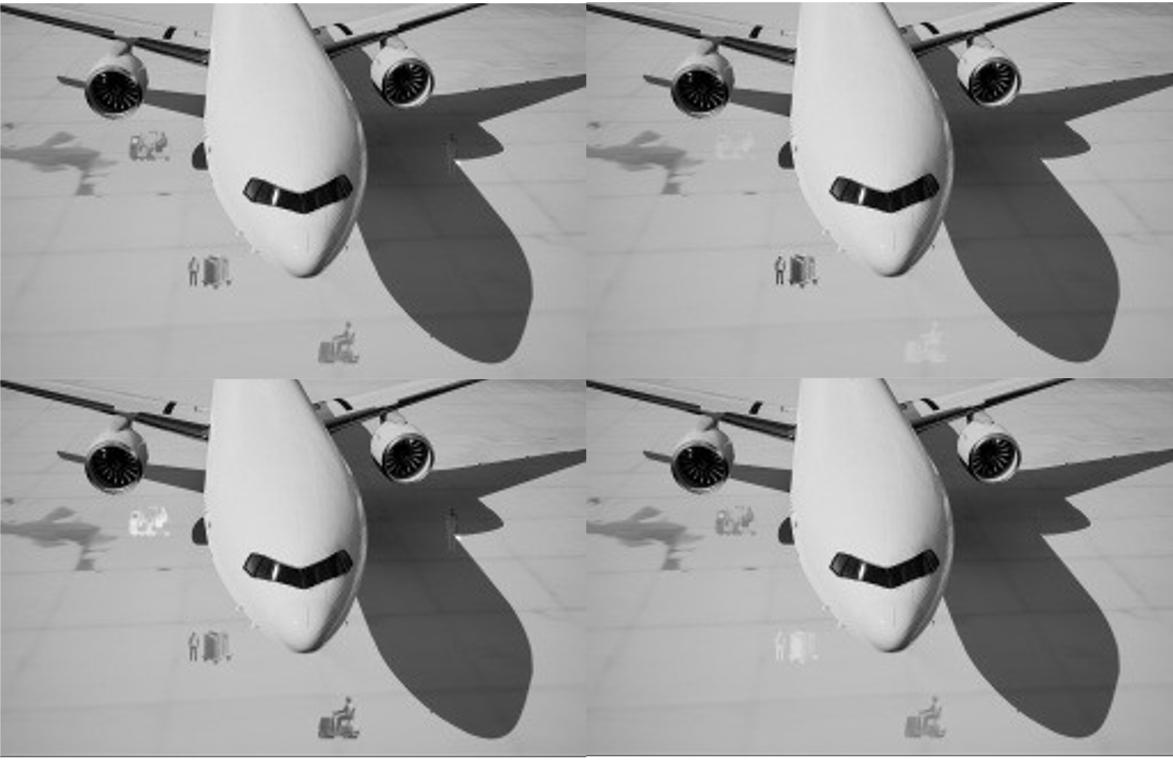}}
      \vspace{-10pt}
     \caption{(a) Various vehicles, personnel and objects considered as local perturbations. (b)-(c) overlapping images with perturbations highlighted within {\color{red}red} circles. (c) non-overlapping images derived from the same seed image. (d) Sampled images from the convex hull, composed of non-overlapping images as shown in~\ref{fig:non-overlapping_airplane}, that  undergo color normalization and display objects in varying shades of gray.}
    \label{fig:airplanes}
    \vspace{-10pt}
 \end{figure}

 \begin{figure}[!t]
    \centering
     \subfigure[Skewness]{
      \label{fig:skewness}
      \includegraphics[width=\linewidth,  trim={0cm 1.2cm 0cm 0cm}, clip]{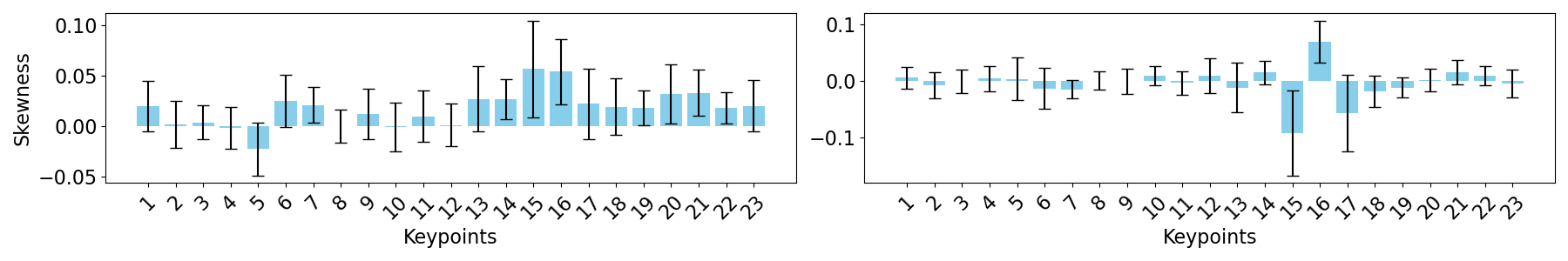}}
       \subfigure[Disparity between mean and median values]{
      \label{fig:mm}
      \includegraphics[width=\linewidth, trim={0cm 1.2cm 0cm 0cm}, clip]{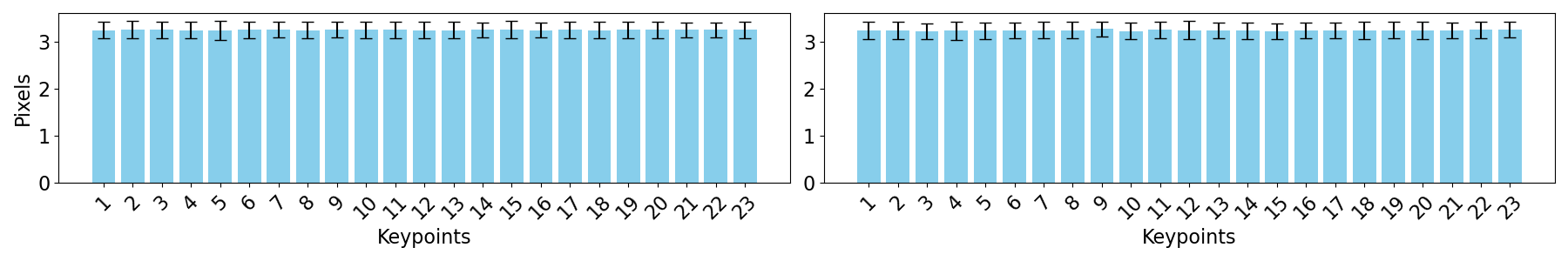}}
       \subfigure[Disparity between mean and mode values]{
      \label{fig:mh}
      \includegraphics[width=\linewidth, trim={0cm 1.2cm 0cm 0cm}, clip]{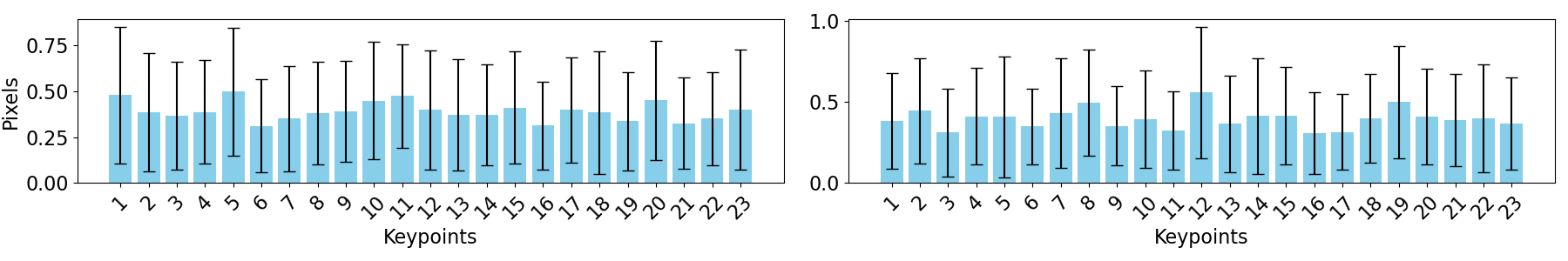}} 
       \subfigure[IQR]{
      \label{fig:iqr}
      \includegraphics[width=\linewidth, trim={0cm 1.2cm 0cm 0cm}, clip]{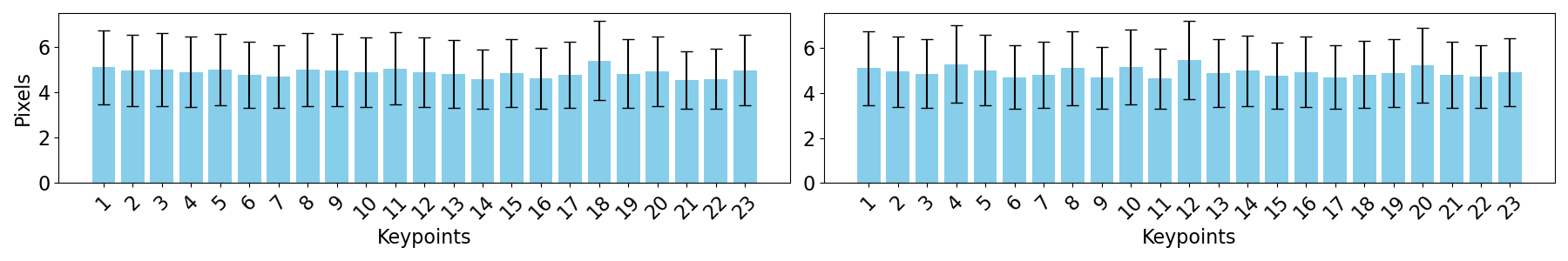}}       
    \caption{Skewness, the disparity between mean and median values, the disparity between mean and mode values, and the IQR for the normalized heatmaps generated by ResNet-18, with row (left) and column (right) dimensions depicted separately in each subfigure.}
    \label{fig:indicators}
 \end{figure}


\paragraph{2. Average pooling} Consider that the parameters for average pooling are tailored based on optimally allocated error thresholds that may vary across different keypoints, thereby necessitating distinct pooling parameters for each. To address this issue, the second component includes the division of a multi-channeled unnormalized heatmap layer into individual channels. Each of these channels is then processed through its own average pooling layer. This design allows for simultaneous verification of all keypoints, eliminating the need to verify each keypoint individually.

\paragraph{3. Flattening and concatenation.} The third component involves the flattening of the previously splitted channels, which are then concatenated into a single-dimensional format for verification.

\subsubsection{Perturbations}
 Out of 7000 images, we randomly sample 200 images as seed images, adding local or global perturbations to them.
\paragraph{Local object occlusions}
We created a set of 40 realistic semantic disturbances featuring personnel and vehicles typically encountered at an airport. These disturbances are depicted in Fig.~\ref{fig:asset}. To create perturbed images, we randomly selected 20 objects as patches, each up to 150 pixels in size, positioned randomly on the seed images, with each perturbed image having one patch. The perturbed images are classified into two groups: overlapping and non-overlapping, based on whether the patches overlap with the airplane in the image. There are 4000 perturbed images in total, including 893 overlapping and 3107 non-overlapping images. The convex hull's complexity is adjusted by changing the number of perturbed images \(m\), where \(m\) ranges from 2 to 4. These perturbed images are randomly selected, allowing us to systematically assess how the addition of semantic disturbances affects the robustness and performance of the system. Note that the convex hull is comprised solely of either overlapping or non-overlapping images. A collection of perturbed images and sampled images from the convex hull are presented in Fig.~\ref{fig:airplanes}.

\paragraph{Local block occlusions} Given a seed image, we generate 4 perturbed images. In each perturbed image, 
a $3 \times 3$ square is placed either away from the airplane or centered over a randomly selected keypoint to create non-overlapping or overlapping images, respectively. All pixel values within this square are randomly selected from the range [0, 255]. A convex hull is formed using these 5 images per seed image, of which all perturbed images are non-overlapping or overlapping ones.

\paragraph{Global perturbations}
We change each pixel's value through two types of global perturbations: brightness and contrast. For brightness, a variation value $b \in \mathbb{Z}$ is applied such that each pixel's value increases by $b$, that is, $I' = \texttt{clip}(I + b)$, where $I$ represents the original pixel values, $I'$ the new pixel values, and $\texttt{clip}$ ensures the values remain within the range [0, 255]. For contrast perturbation, a variation value $c \in \mathbb{R}$ adjusts each pixel's value by a percentage $c$, formulated as $I' = \texttt{clip}(I \times (1 + c))$.  The convex hull is constructed as follows. For a positive value $b \in \mathbb{Z}_+$, two perturbed images are created for $b$ and $-b$, respectively. These images act as vertices of the convex hull. Along with the seed image, this approach facilitates verification of the model's robustness to any brightness variation within the range $[-b, b]$. The same methodology applies to contrast variations $c \in \mathbb{R_+}$. We examine the effects for $b$ values of $\{1, 2\}$ and $c$ values of $5\times 10^{-4}, 5\times 10^{-3}, 1\times 10^{-2}\}$.



\subsection{Validity of Assumption~\ref{asmp:gaussian}}\label{sec:asmp}

{\color{\modifycolor}To assess the degree of symmetry and unimodality exhibited by the normalized heatmaps, we use four metrics: skewness, the disparity between mean and median, the disparity between mean and mode (peak), and the Interquartile Range (IQR), which is the difference between the third quartile (Q3) and the first quartile (Q1). A skewness near zero and a minimal difference between the mean and median indicate a symmetrical distribution. Combined with small skewness and close mean and median, a small IQR and close mean and mode suggest unimodality. These metrics are calculated individually for each image dimension. The evaluation uses the ResNet-18-based backbone model  across 3000 seed images and 4000 perturbed images with object occlusions and is presented in Fig.~\ref{fig:indicators}. The findings reveal skewness values approximately zero and mean and median values that are closely aligned, especially when considering the heatmap dimensions of $256 \times 256$. The difference between mean and mode and IQR values are also minor relative to the total range of 256. Thus, these statistical outcomes demonstrate that the normalized heatmaps predominantly feature unique peaks and exhibit axis symmetry, supporting Assumption~\ref{asmp:gaussian}.}


\begin{figure}[!t]
    \centering
      \includegraphics[width=\linewidth]{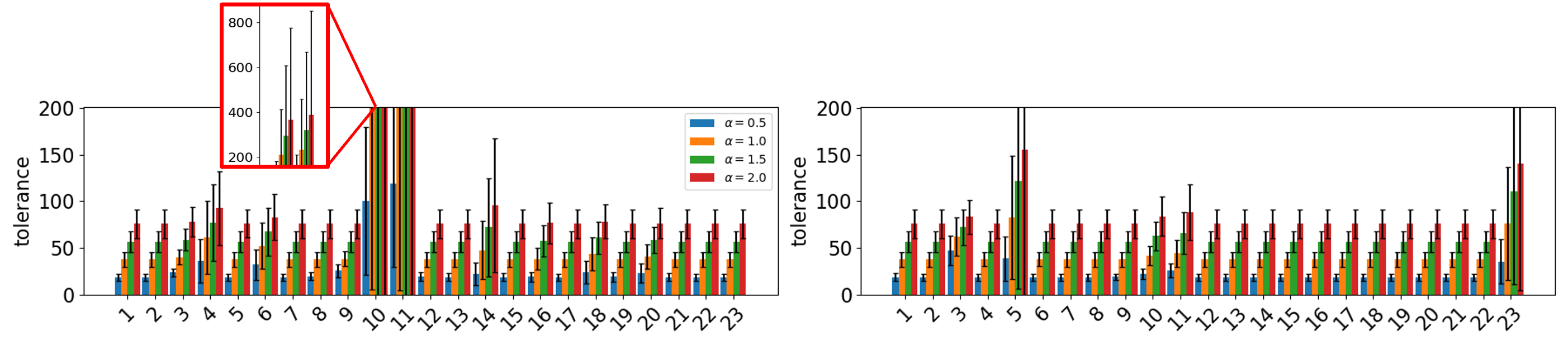}
     \caption{Results about the optimal error threshold allocation for horizontal (left) and vertical (right) dimensions.}
    \label{fig:deviation}
    \vspace{-10pt}
 \end{figure}

\subsection{Probabilistic Soundness and Completeness of Optimal Error Threshold Allocation}\label{sec:prob_property}
\subsubsection{Optimal tolerance allocation}
Note that the outcomes of allocating optimal tolerance are influenced by images and defined thresholds for acceptable pose errors. In our experiments, we manipulate these thresholds by incorporating a scaling factor $\alpha \in \mbR_+$, resulting in adjusted error thresholds $\alpha \bm{\epsilon}_r$ and $\alpha \bm{\epsilon}_t$.  The pose error thresholds are set as $\bm{\epsilon}_r = \alpha \cdot [10^\circ, 10^\circ, 10^\circ]$ and $\bm{\epsilon}_t = \alpha \cdot [4, 4, 20]$. The threshold factor $\alpha$ is varied across the values $\{0.5, 1.0, 1.5, 2.0\}$. We maintain the scaling factor $\kappa$ at 1.0 as specified in Eq.~\eqref{eq:formulation}, and assign the weights $w_1$ and $w_2$ values of 1 and 5, respectively. Our analysis includes 3000 images, each contributing 5000 sets of perturbed keypoints, accumulating a total of $1.5 \times 10^7$ samples. The mean and standard deviation of the tolerance allocated in the horizontal and vertical directions are presented in Fig.~\ref{fig:deviation}. As the allowable pose errors grows, the tolerance allocated per keypoint increases, predominantly uniform in both the horizontal and vertical axes, with the exception of the 10-th and 11-th keypoints, which are symmetrical with respect to the body axis of the plane (see Fig.~\ref{fig:airplane}), exhibit larger tolerances horizontally, while the 5-th and 23-th keypoints, aligned along the body axis, have larger tolerances vertically.

\subsubsection{Probabilistic soundness} In this part, we intend to evaluate the probabilistic soundness of optimal error threshold allocation by verifying if the pose estimation errors resulting from perturbed keypoints fall within the predefined error thresholds. One approach is to perform random sampling within the hyper-rectangle $\ccalH\ccalR (\delta \bbv^*)$. However, given the high dimensionality of $\ccalH\ccalR$, which is $2^{2K}$ with $K$ exceeding 20 in our scenario, this method demands an extraordinarily large sample size to achieve sufficient coverage. To address this issue, we only evaluate on vertices. For a seed image we randomly choose vertices from the hyper-rectangle $\ccalH\ccalR (\delta \bbv^*)$, which we then add to the coordinates of the ground-truth keypoints to create perturbed keypoints. Mathematically, in matrix form,
$
    \hat{\bbV} = \bbA \odot \delta \bbV^* + \bbV, 
$
where the elements of matrix $\bbA$ are randomly set to either -1 or 1, the symbol $\odot$ represents element-wise multiplication, and $\hat{\bbV}$ denotes the perturbed keypoints. This allows the perturbation of keypoints to reach the maximum tolerable errors. 

The simulation setup is largely the same as the previous section. We vary both the scaling factor $\kappa$ in Eq.~\eqref{eq:formulation} and the threshold factor $\alpha$, and compute the ratio of samples where the pose estimation errors remain within thresholds over the total number of samples.  We report the ratio of samples resulting in tolerable pose errors. The results are summarized in Table~\ref{tab:prob_soundness}. As we can see, only a small number of samples result in poses that surpass the error threshold. There is a noticeable trend where, vertically, as $\kappa$ increases, the ratio of acceptable samples rises due to the shrinking of the keypoint error threshold polytope. Horizontally, increasing $\alpha$ results in a lower ratio, as linearization becomes less accurate when the pose moves away from the reference point.
\renewcommand{\arraystretch}{1.2} 
\begin{table}[!t]
\setlength{\tabcolsep}{4pt} 
\centering\footnotesize
\begin{tabular}{c|cccc||cccc}
\bhline
\multirow{2}{*}{\diagbox{factor $\kappa$}{factor $\alpha$}} & \multicolumn{4}{c||}{soundness}  & \multicolumn{4}{c}{complenteness}  \\ 
& 0.5 & 1.0 & 1.5  & 2.0  & 0.5 & 1.0 & 1.5  & 2.0 \\
\hline
1.0 & 1.0 & 0.998978 & 0.996596 & 0.988339 & 1.93$\times 10^{-6}$ & 3.07$\times 10^{-5}$ & 4.85$\times 10^{-5}$ & 6.20$\times 10^{-5}$\\
1.5 & 1.0 & 0.999993 & 0.999887 & 0.999890 & $<$ 6.67 $\times 10^{-9}$  & 1.00$\times 10^{-6}$ & 2.00$\times 10^{-6}$ & 3.67$\times 10^{-6}$\\ 
2.0 & 1.0 & 1.0 & 1.0 & 0.999999 & $<$ 6.67 $\times 10^{-9}$  & $<$ 6.67 $\times 10^{-9}$  & 3.33$\times 10^{-7}$ & 6.67$\times 10^{-7}$\\ 
\bhline
\end{tabular}
\caption{Probabilistic soundness and completeness of optimal error threshold allocation. Ideal value is 1.0.}
\label{tab:prob_soundness}
\vspace{-20pt}
\end{table}

\begin{figure}[!t]
    \centering
     \subfigure[Keypoint 1.]{
      \label{fig:kp1}
      \includegraphics[width=0.3\linewidth, trim={2cm 0cm 2cm 0cm}, clip]{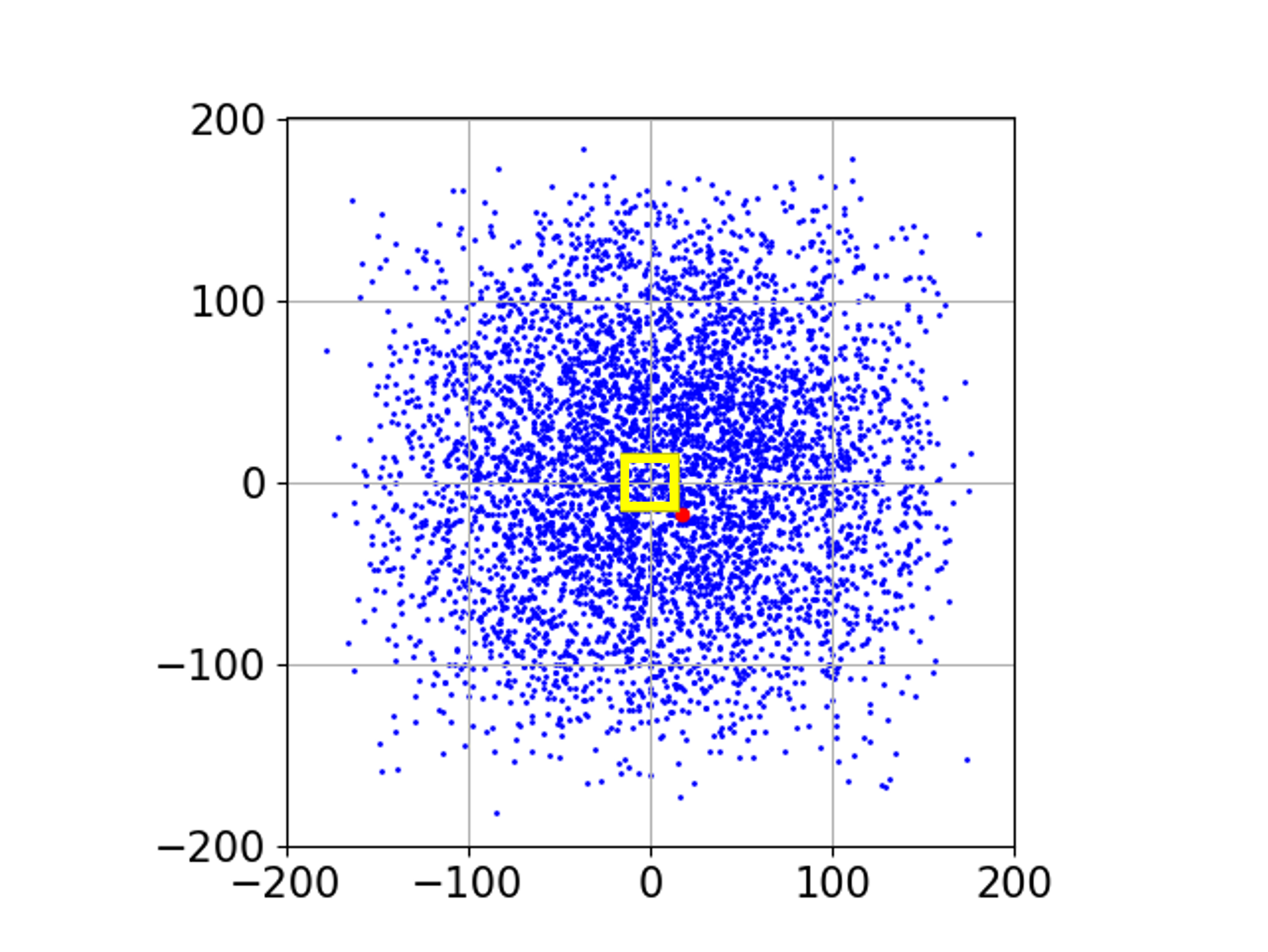}}
     \subfigure[Keypoint 3.]{
      \label{fig:kp3}
      \includegraphics[width=0.3\linewidth, trim={2cm 0cm 2cm 0cm}, clip]{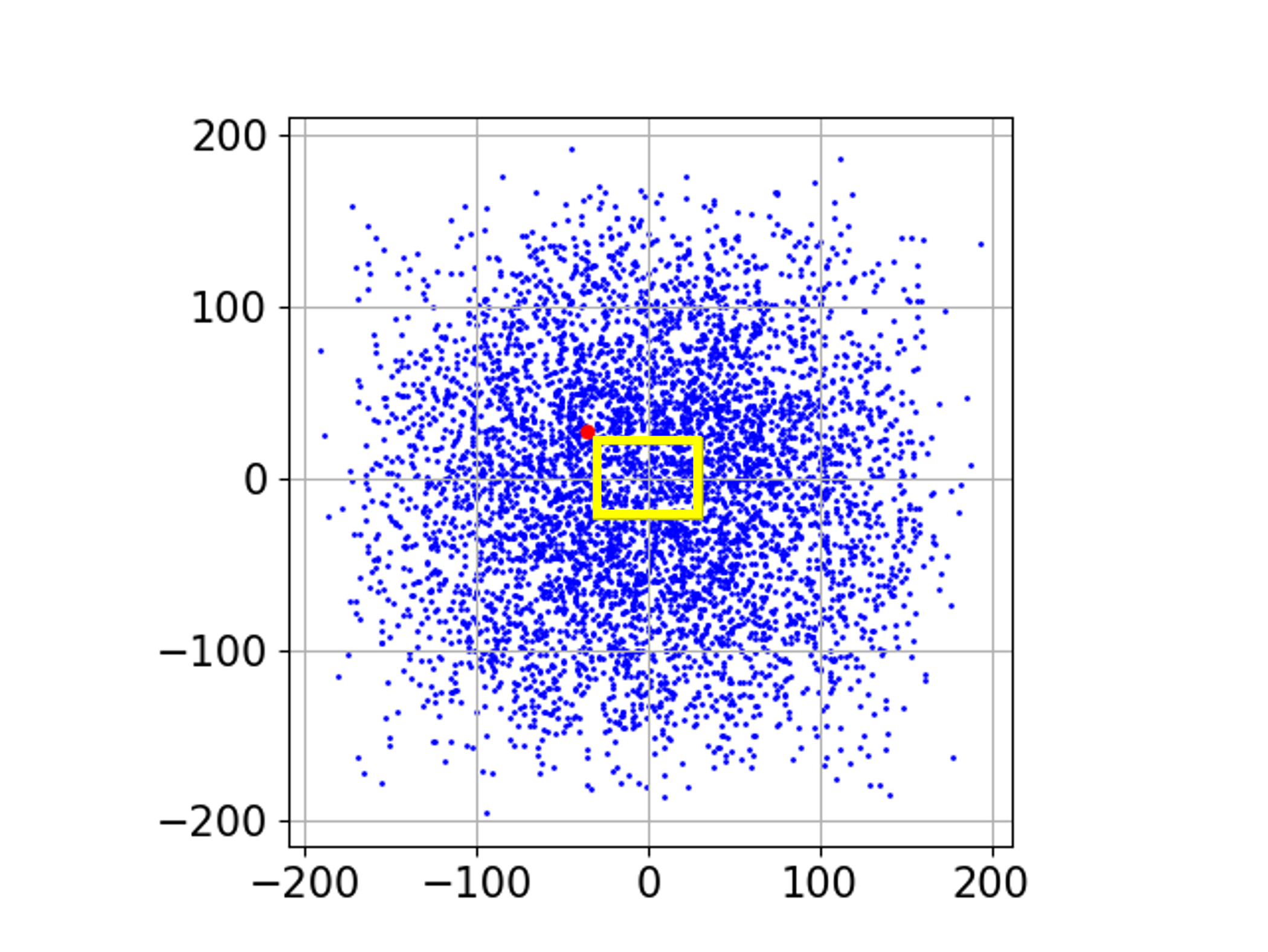}}
     \subfigure[Keypoint 5.]{
      \label{fig:kp5}
      \includegraphics[width=0.3\linewidth, trim={2cm 0cm 2cm 0cm}, clip]{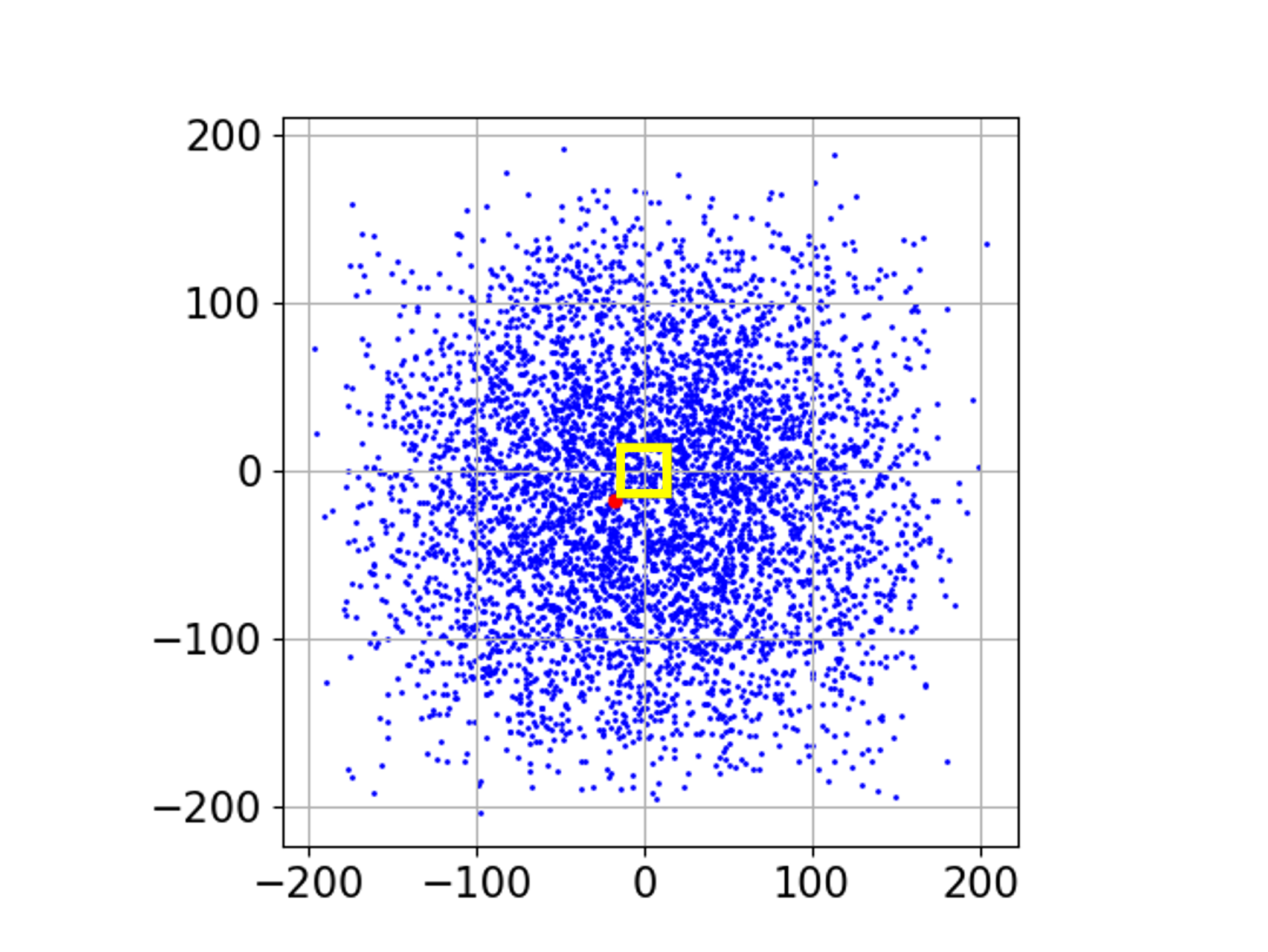}}
     \caption{Unscaled reprojection errors ({\color{blue}blue}) from 5000 samples and the optimally allocated error thresholds ({\color{yellow}yellow} square) for keypoints 1, 3 and 5. {\color{\modifycolor} The {\color{red}red} points denote a keypoint error allocation that collectively results in the violation of pose error bounds, obtained by deviating outward from the yellow boundaries.}}
    \label{fig:kps}
 \end{figure}
\subsubsection{Probabilistic completeness} 
To compute completeness, we randomly sample from the set of tolerable pose errors $\Xi$, use the perspective projection model~\eqref{eq:projection} to determine keypoints, and then compute the reprojection error relative to the ground truth. We verify whether these reprojection errors remain within the permissible error bounds on keypoints, represented by the hyper-rectangle $\ccalH\ccalR$. This simulation follows the same setup as used for probabilistic soundness. The results are presented in Table~\ref{tab:prob_soundness}. Vertically, as $\kappa$ increases, completeness diminishes due to the reduction in the size of $\ccalH\ccalR$, leading to more samples exceeding these bounds. On the other hand, there is no discernible trend when viewed horizontally.

When examining Table~\ref{tab:prob_soundness}, a trade-off between soundness and completeness is evident. Our framework exhibits better soundness, implying that the set of allocated keypoint errors, $\ccalH\ccalR$, is relatively small compared to the ground-truth set $\delta \bbV_\Xi$. This smaller size of $\ccalH\ccalR$ accounts for the reduced completeness observed in the results. To support this, Fig.~\ref{fig:kps} illustrates the reprojection errors relative to the optimally allocated bounds (rectangle) for keypoints 1, 3 and 5 from a specific image under the conditions $\alpha = 0.5$ and $\kappa = 1.0$. These reprojection errors are calculated by randomly generating pose errors within tolerable ranges and calculating the difference between ground-truth keypoints and those obtained from perspective projection model~\eqref{eq:projection}. The illustration shows that a significant number of samples exceed the computed bounds, contributing to the overall low completeness observed across all keypoints. Therefore, if the verification tool returns a hold, it implies that the pose estimation is robust with high probability. On the other hand, if the tool returns a violation, we cannot draw any definitive conclusions. {\color{\modifycolor}The conservativeness in error allocation mainly stems from the need to decouple the dependencies among keypoints.}


\subsection{Verification of Local and Global Perturbations}\label{sec:verif_res}
In this section, we assess the robustness of the pose estimation method when subjected to various levels of perturbations. We aim to answer three key questions:
\begin{enumerate}
    \item How computationally efficient is the resulting neural network verification problem?
    \item How accurate is the proposed certification method for robust pose estimation?
\end{enumerate}

\subsubsection{Metric}

We measure the performance using verification times and verified rates. Verified rate is defined as the proportion of cases where the verification algorithm confirms robustness against those where seed images produce acceptable pose estimation errors.   {\color{\modifycolor}Another critical measure is verification accuracy which is defined as the proportion of cases where the verification algorithm confirms robustness against those that are indeed robust. However, determining the exact number of truly robust instances is impractical. Consequently, the verified rate serves as a lower bound of verification accuracy as the instances with acceptable pose estimation errors from seed images exceed those that are truly robust}.

\subsubsection{Verification results}

{\color{\modifycolor} We employ the verification toolbox \texttt{ModelVerification.jl} (\texttt{MV})~\cite{wei2024modelverification}, which accepts convex hulls as input specifications.~\mv is the state-of-the-art verifier that supports a wide range of verification algorithms and is the most user-friendly to extend. It follows a branch-and-bound strategy to divide and conquer the problem efficiently. Two parameters guide this process: \texttt{split\_method} determines the division of an unknown branch into smaller branches for further refinement, and \texttt{search\_method} dictates the approach to navigating through the branch. We set \texttt{search\_method} to use breadth-first search and  \texttt{split\_method} to bisect the branch. The computing platform is a Linux server equipped with an Intel CPU with 48 cores running at 2.20GHz and 376GB of total memory, approximately 150GB of which is available owing to multiple users. Additionally, the server includes 4 NVIDIA RTX A4000 GPUs, each with 16GB of memory.

\renewcommand{\arraystretch}{1.2} 
 \begin{table}[!t]
     \centering\footnotesize
      {\color{\modifycolor}
     \begin{tabular}{c||cc|cc||cc|cc}
     \bhline
      \multirow{3}{*}{$m$} & \multicolumn{4}{c||}{non-overlapping images}  &  \multicolumn{4}{c}{overlapping images} \\
     \cline{2-9}
      & \multicolumn{2}{c|}{$\kappa = 1.0$} &  \multicolumn{2}{c||}{$\kappa = 1.5$} & \multicolumn{2}{c|}{$\kappa = 1.0$} & \multicolumn{2}{c}{$\kappa = 1.5$} \\
     \cline{2-9}
         & $\alpha = 1.0$ & $\alpha = 1.5$  & $\alpha = 1.0$ & $\alpha = 1.5$ & $\alpha = 1.0$ & $\alpha = 1.5$  & $\alpha = 1.0$ & $\alpha = 1.5$\\
     \hline
       2  & 38.3$\pm$56.7 &  \cellcolor{gray!20}  {\bf 11.9}$\pm$11.1 & 81.5$\pm$79.7 & 47.1$\pm$69.0 & 82.1$\pm$128.3 & 31.1$\pm$29.4 & 133.1$\pm$127.3 & 72.4$\pm$79.6 \\
        3  & 63.3$\pm$89.2 & \cellcolor{gray!20}  {\bf 16.9}$\pm$12.3 & 104.9$\pm$188.3 & 60.7$\pm$74.0 & 87.9$\pm$87.4 & 47.7$\pm$44.7 & 154.8$\pm$127.4 & 98.3$\pm$96.2 \\
        4  & 73.1$\pm$103.8 & \cellcolor{gray!20}  {\bf 23.2}$\pm$16.5 & 127.7$\pm$105.4  & 79.1$\pm$101.0 & 109.2$\pm$101.8 & 60.4$\pm$36.7 & 184.3$\pm$132.5 & 110.1$\pm$87.7\\
     \bhline
     \end{tabular}
     }
     \caption{Statistical results on verification time for local perturbations (seconds).}
     \label{tab:time_local_perturb}
     \vspace{-10pt}
 \end{table}

\renewcommand{\arraystretch}{1.2} 
 \begin{table}[!t]
     \centering\footnotesize
      {\color{\modifycolor}
     \begin{tabular}{c||cc|cc||cc|cc}
     \bhline
      \multirow{3}{*}{$m$} & \multicolumn{4}{c||}{ non-overlapping images}  &  \multicolumn{4}{c}{overlapping images} \\
     \cline{2-9}
       & \multicolumn{2}{c|}{$\kappa = 1.0$} &  \multicolumn{2}{c||}{$\kappa = 1.5$} & \multicolumn{2}{c|}{$\kappa = 1.0$} & \multicolumn{2}{c}{$\kappa = 1.5$} \\
     \cline{2-9}
         & $\alpha = 1.0$ & $\alpha = 1.5$  & $\alpha = 1.0$ & $\alpha = 1.5$ & $\alpha = 1.0$ & $\alpha = 1.5$  & $\alpha = 1.0$ & $\alpha = 1.5$\\
     \hline
        2  & 55.0\% & \cellcolor{gray!20}  {\bf 88.5}\% & 16.9\%  & 55.6\% & 42.2\%  &  74.1\% & 13.1\% & 44.8\%  \\
        3  & 55.8\% & \cellcolor{gray!20}  {\bf 88.5}\% & 16.6\%  & 55.3\% & 36.7\% & 65.1\% & 12.6\%  & 38.9\% \\
        4  & 55.3\% & \cellcolor{gray!20}  {\bf 85.5}\% & 16.7\%  & 53.1\% & 28.0\% & 50.4\% & 10.6\% & 29.4\% \\
     \bhline
     \end{tabular}
     }
     \caption{Statistical results on verified rate for local perturbations.}
     \label{tab:rate_local_perturb}
     \vspace{-10pt}
 \end{table}

\paragraph{Local object occlusions for the CNN-based model} The results presented in Tab.~\ref{tab:time_local_perturb} indicate that for non-overlapping images, verification time increases as the number $m$ of perturbed images forming the convex hull rises. As the error bounds for pose estimation expand, with threshold factor $\alpha$ increasing from 1.0 to 1.5, the time decreases. A similar effect is observed when the scaling factor $\kappa$ decreases from 1.5 to 1.0. This reduction occurs because the optimally allocated error thresholds on keypoints expand with increasing $\alpha$ and decreasing $\kappa$, as illustrated in Fig.~\ref{fig:deviation}, which results in fewer nodes. A similar trend is observed for overlapping images. We emphasize that the verification process for overlapping images requires more time than for non-overlapping ones, given the same number of perturbed images and identical pose estimation error bounds, indicating that the keypoint detection model is effective in suppressing disturbances external to the airplane. 

In reference to the verified rates displayed in Tab.~\ref{tab:rate_local_perturb}, for non-overlapping images, the rate remains stable regardless of the number of perturbed images, given the same $\kappa$ and $\alpha$. Conversely, there is an increase in the verified rate with a decrease in $\kappa$ and an increase in $\alpha$, as larger allocated keypoint error thresholds or larger allowable pose error thresholds  result in more images being verified as robust. In the case of overlapping images, a notable trend is the decline in the verified rate when the number of perturbed images increases. This is because a rise in the number of perturbed images means more objects are overlaid on the airplane, which compromises the accuracy of predictions and, in turn, decreases the number of images verified as robust.

\paragraph{Global perturbations for the CNN-based model} A similar trend to that seen with local perturbations emerges, as indicated in Tables~\ref{tab:time_global_perturb} and~\ref{tab:rate_global_perturb}. The verification can handle contrast variations $c$ of 1\% and brightness variations $b$ of 2/255. Greater variations in contrast and brightness lead to longer verification times, whereas a larger $\alpha$ reduces verification time and increases the verified rate. }

\paragraph{Local block occlusions for the ResNet-18-based model} We set $\kappa = 1.0$ and $\alpha = 1.5$. For the convex hull comprised of non-overlapping images, the verification time is 314.4$\pm$227.0 in seconds, with a verification rate of 93.5\%. Conversely, for the convex hull consisting of overlapping images, the verification time significantly escalates to 1571.7$\pm$1213.5 seconds, and the verification rate is 94.0\%.

\begin{table}[!t]
\setlength{\tabcolsep}{3.5pt} 
     \centering\footnotesize
      {\color{\modifycolor}
     \begin{tabular}{c||cc|cc||c||cc|cc}
     \bhline
       \multirow{2}{*}{$c$} &  \multicolumn{2}{c|}{$\kappa = 1.0$} &  \multicolumn{2}{c||}{$\kappa = 1.5$}  &   \multirow{2}{*}{$b$}  & \multicolumn{2}{c|}{$\kappa = 1.0$} & \multicolumn{2}{c}{$\kappa = 1.5$} \\
     \cline{2-5}
     \cline{7-10}
       & $\alpha = 1.0$ & $\alpha = 1.5$ & $\alpha = 1.0$ & $\alpha = 1.5$ &  & $\alpha = 1.0$ & $\alpha = 1.5$ & $\alpha = 1.0$ & $\alpha = 1.5$ \\
     \hline
        0.05\%  & 79.7$\pm$105.5 & \cellcolor{gray!20}  {\bf 15.4}$\pm$4.9 & 133.9$\pm$95.5 & 83.2$\pm$105.0  & 1 & 135.0$\pm$171.5 & \cellcolor{gray!20}  {\bf 32.7}$\pm$12.4 & 236.4$\pm$164.0  &  144.7$\pm$164.5\\
       0.5\%  & 88.6$\pm$112.3  & \cellcolor{gray!20}  {\bf 19.6}$\pm$5.7 & 159.0$\pm$111.1 & 92.0$\pm$108.1  & 2 & 223.4$\pm$301.9  & \cellcolor{gray!20}  {\bf 64.6}$\pm$22.5 & 373.5$\pm$309.2 & 225.1$\pm$266.0 \\
       1\%  & 176.3$\pm$232.9 & \cellcolor{gray!20}  {\bf 48.7}$\pm$17.7 & 304.1$\pm$243.6 & 200.0$\pm$249.4 & --  & -- & -- & --  & -- \\
    \bhline
     \end{tabular}
     }
     \caption{Statistical results on verification time for global perturbations (seconds).}
     \label{tab:time_global_perturb}
     \vspace{-10pt}
 \end{table}



 \begin{table}[!t]
     \centering\footnotesize
      {\color{\modifycolor}
      \begin{tabular}{c||cc|cc||c||cc|cc}
     \bhline
  \multirow{2}{*}{$c$} &  \multicolumn{2}{c|}{$\kappa = 1.0$} &  \multicolumn{2}{c||}{$\kappa = 1.5$}  &   \multirow{2}{*}{$b$}  & \multicolumn{2}{c|}{$\kappa = 1.0$} & \multicolumn{2}{c}{$\kappa = 1.5$} \\
     \cline{2-5}
     \cline{7-10}
       & $\alpha = 1.0$ & $\alpha = 1.5$ & $\alpha = 1.0$ & $\alpha = 1.5$ &  & $\alpha = 1.0$ & $\alpha = 1.5$ & $\alpha = 1.0$ & $\alpha = 1.5$ \\
     \hline
        $5\times10^{-4}$  & 57.0\% & \cellcolor{gray!20}  {\bf 91.5}\% & 17.3\%  & 57.1\% & 1 & 57.0\% & \cellcolor{gray!20} {\bf 91.5\%} & 15.0\% &  57.7\%\\
        $5\times10^{-3}$  & 56.5\% & \cellcolor{gray!20}  {\bf 90.5}\% & 16.8\% & 56.7\% & 2 & 56.0\% & \cellcolor{gray!20} {\bf 88.5\%} & 14.8\% &  57.6\%\\
        $1\times10^{-2}$  & 56.0\% & \cellcolor{gray!20}  {\bf 87.4}\% & 17.3\%  & 56.8\% & -- & & & --  & -- \\
    \bhline
     \end{tabular}
     }
     \caption{ Statistical results on verified rate for global perturbations.}
     \label{tab:rate_global_perturb}
     \vspace{-10pt}
 \end{table}

\section{Conclusions}
{\color{\modifycolor}In this study, we introduce a framework designed to certify the robustness of learning-based keypoint detection and pose estimation methods. Given system-level requirements, our approach transforms the certification of PnP-based pose estimation into the standard verification for classification, allowing us to leverage off-the-shelf tools. The evaluation results demonstrated that our framework can handle realistic semantic perturbations compared to existing methods. We emphasize that our certification framework is general for safety-critical applications that depend on accurate keypoint detection. These include airport runway detection for automatic landing, pedestrain pose estimation for autonomous driving, and anatomical landmark identification for robot-assisted surgery. Future directions of this framework could include: 1)	Expanding the input specifications to represent more perturbations, such as the translational movement of objects. 2) Reducing the conservativeness caused by independent error allocation among keypoints.}

\section*{Acknowledgement}
This material is based upon work supported by The Boeing Company. Any opinions, findings, and conclusions or recommendations expressed in this material are those of the authors and do not necessarily reflect the views of The Boeing Company.

\bibliographystyle{ACM-Reference-Format}
\bibliography{lit}

\end{document}